\documentclass{article} 
\usepackage{iclr2027_conference,times}

\usepackage{amsmath,amsfonts,bm}

\def\eqref#1{equation~\ref{#1}}

\def\1{\bm{1}}

\DeclareMathAlphabet{\mathsfit}{\encodingdefault}{\sfdefault}{m}{sl}
\SetMathAlphabet{\mathsfit}{bold}{\encodingdefault}{\sfdefault}{bx}{n}

\usepackage[utf8]{inputenc} 
\usepackage[T1]{fontenc}    
\usepackage{hyperref}       
\usepackage{url}            
\usepackage{booktabs}       
\usepackage{amsfonts}       
\usepackage{nicefrac}       
\usepackage{microtype}      
\usepackage{xcolor}         
\usepackage{amsmath}
\usepackage{graphicx}
\usepackage{subcaption}
\usepackage{algorithm}
\usepackage{algpseudocode}
\usepackage{amsthm}
\usepackage{multirow}
\usepackage{dsfont}
\usepackage{wrapfig}
\usepackage{xspace}
\usepackage[table]{xcolor}

\newtheorem{proposition}{Proposition}

\newcommand{\methodbase}{DICE}

\newcommand{\ourmethod}{\texttt{\methodbase}\xspace}
\newcommand{\ourmethodnoreg}{\texttt{\methodbase} (w/o reg)\xspace}

\definecolor{lightindigo}{rgb}{0.533, 0.349, 0.671}
\definecolor{royalblue}{rgb}{0.2549,0.4118,0.8824}

\hypersetup{
    colorlinks=true,  %
    citecolor=royalblue, 
    linkcolor=royalblue,
    filecolor=royalblue,
    urlcolor=royalblue,
}

\title{Uncertainty Estimation in Pathology Foundation Models via Deep Mutual Learning}

\author{%
    Gbègninougbo Aurel Davy Tchokponhoue \thanks{Equal contribution.} \\
    UM6P, Ben Guerir, Morocco\\
  \texttt{aurel.tchokponhoue@um6p.ma} \\
  \And
  Sevda Öğüt \footnotemark[1] \\
    EPFL, Lausanne, Switzerland\\
  \texttt{sevda.ogut@epfl.ch} \\
  \And
  Ali Idri \\
    UM5, Rabat, Morocco\\
  \texttt{ali.idri@ensias.um5.ac.ma} \\
  \And
  Dorina Thanou \\
    EPFL, Lausanne, Switzerland\\
  \texttt{dorina.thanou@epfl.ch} \\
  \And
  Pascal Frossard\\
    EPFL, Lausanne, Switzerland\\
  \texttt{pascal.frossard@epfl.ch} 
}

\iclrfinalcopy 
\begin{document}

\maketitle

\begin{abstract}
Pathology foundation models (PFMs) offer generalizable representations for whole-slide image (WSI) analysis, yet their clinical adoption remains limited. 
Specifically, their predictions lack reliable confidence estimates, and no single PFM is universally best across tasks, which severely undermines trust in medical settings.
To overcome this, we propose \ourmethod, a plug-and-play framework that ensembles $K$ frozen PFMs and estimates uncertainty based on their consensus.
We align the ensemble members via deep mutual learning and theoretically show that this objective controls an upper bound on epistemic uncertainty.
Additionally, we demonstrate that the ensemble localizes abnormalities at the patch level without any explicit supervision.
We evaluate \ourmethod on three challenging WSI benchmarks.
Notably, our framework provides reliable uncertainty estimates that accurately flag failure-prone cases under in- and out-of-distribution settings, while matching or outperforming SOTA baselines in classification, calibration, and localization.
Overall, \ourmethod takes a crucial step toward translating PFMs into uncertainty-aware decision-support systems.
\end{abstract}

\section{Introduction}

Histopathology is a central component of clinical decision-making, as it provides information on cellular morphology and tissue architecture that is essential for diagnosis, prognosis, and treatment planning.
In particular, hematoxylin and eosin (H\&E)-stained tissue samples are routinely collected and digitized, resulting in vast amounts of gigapixel whole-slide images (WSIs), which are typically divided into patches for analysis.
Together with advances in large-scale computation, the abundance of WSI data has enabled the development of pathology foundation models (PFMs)~\citep{chen2024uni, xu2024whole, zimmermann2024virchow2}.
These models are pretrained using self-supervised learning (SSL) and can be used to extract general-purpose representations.
The use of these representations for various downstream tasks, including survival analysis~\citep{yang2025foundation}, biomarker discovery~\citep{campanella2025real}, and tumor subtyping~\citep{ougut2026graphist}, has demonstrated their strong generalization capabilities.

However, the translation of such foundation models into routine clinical practice remains limited~\citep{aggarwal2025artificial, jaume2026foundation}.
A critical gap is that PFM pipelines generally output predictions without any explicit measure of confidence.
This is a major shortcoming for clinical deployment, where models should not only be accurate but also indicate when their predictions may be uncertain~\citep{begoli2019need}.
Furthermore, selecting an appropriate PFM for deployment is itself challenging, as model performance varies substantially across downstream tasks, datasets, and clinical contexts~\citep{campanella2025clinical, neidlinger2025benchmarking}.
As a result, there is rarely a universally superior PFM, highlighting the need for methods that quantify model reliability and support informed model selection.

To address these limitations, we propose \ourmethod (\underline{D}isagreement-\underline{I}nformed \underline{C}oordination of \underline{E}xperts), a plug-and-play framework that integrates a diverse ensemble of PFMs, spanning different architectures, pretraining objectives, and data scales, for informed WSI analysis.
The core intuition behind \ourmethod is that disagreement among heterogeneous PFMs is an informative signal for uncertainty estimation.
Accordingly, \ourmethod is trained by treating each PFM and its corresponding classification head as an \textit{expert}, and aligning them through deep mutual learning (DML)~\citep{zhang2018deep} as well as Gramian~\citep{cicchetti2025gramian} objectives.
These objectives encourage expert agreement and effectively reduce epistemic uncertainty, which is proven to be upper-bounded by the aggregate DML loss.
At inference, \ourmethod leverages the predictive uncertainty to flag error-prone slides, enabling safe deferral to clinicians in ambiguous cases. 
In addition, \ourmethod exploits expert consensus to better localize patch-level lesions, enabling faster identification of diagnostically relevant tissue regions.
To the best of our knowledge, \ourmethod is the first framework to leverage PFM ensembles for task-agnostic uncertainty estimation in computational pathology.

Empirically, we evaluate our framework on three weakly supervised WSI benchmarks: PANDA~\citep{bulten2022artificial} for prostate cancer grading and CAMELYON16/17~\citep{ehteshami2017diagnostic, litjens20181399} for breast lymph-node metastasis detection.
Across these datasets, the ensemble's predictive entropy identifies uncertain samples more reliably than standard baselines, in both in- and out-of-distribution settings.
Moreover, \ourmethod's fusion-of-experts mechanism matches or exceeds the strongest single PFM baselines across a range of downstream tasks and metrics, while providing unsupervised lesion localization through expert consensus.

Our contributions are as follows:
\begin{itemize}
    \item We propose \ourmethod, the first plug-and-play framework that fuses PFM ensembles for uncertainty-aware decision support in histopathology.
    \item We establish a theoretical foundation for our framework by proving that minimizing the training-time alignment objective contracts a computable upper bound on epistemic uncertainty.
    \item Our extensive experiments demonstrate that our method not only accurately highlights unreliable samples, but also matches or improves classification, calibration, and localization performance compared to SOTA models.
\end{itemize}
Overall, \ourmethod offers a promising step toward making PFMs more trustworthy in clinical settings by establishing a theoretically grounded approach to uncertainty-aware fusion.
\section{Related work}

\paragraph{Pathology foundation models.}
 A growing body of work has produced pathology-specific foundation models, including patch-level ones such as Virchow~\citep{vorontsov2023virchow} and UNI~\citep{chen2024uni}, as well as slide-level ones such as GigaPath~\citep{xu2024whole} and PRISM~\citep{shaikovski2024prism}.
 While these models enable transferable representations for histopathology, they are typically used as frozen feature extractors and do not provide explicit confidence estimates, which is a key downside for clinical use.
 Furthermore, recent benchmarks~\citep{campanella2025clinical, neidlinger2025benchmarking} show that no single PFM is uniformly best across tasks and datasets.
 These limitations motivate methods that can exploit the complementary strengths of multiple PFMs.
 
\paragraph{Fusion methods in histopathology.}
A growing body of concurrent work studies the fusion of PFMs.
ELF~\citep{luo2025ensemble} integrates several PFMs into a unified slide representation through ensemble learning and reports strong downstream performance, while FuseCPath~\citep{yang2025fusion} proposes a multi-scale fusion framework with patch-level re-embedding and slide-level collaborative distillation, and FM2~\citep{yu2025fm2} disentangles consensus and divergence components across foundation models for downstream prediction.
None of these methods, however, provides explicit confidence estimates.
PICTURE~\citep{zhao2025uncertainty}, on the other hand, does quantify uncertainty by combining multiple PFMs with prototype-guided Bayesian inference, deep ensembling, and normalizing-flow-based out-of-distribution detection.
Yet, it is built for a specific task, namely, to differentiate glioblastoma from its mimics.
As a result, none of the above methods offers a task-agnostic framework for uncertainty-aware fusion of PFMs.

\paragraph{Uncertainty estimation and deep mutual learning.}
To fill this gap, we draw inspiration from the broader uncertainty estimation literature and from deep mutual learning, briefly reviewed below.
Deep ensembles~\citep{lakshminarayanan2017simple} provide a scalable alternative to Bayesian deep learning, aggregating independently trained models and using their disagreement as a proxy for uncertainty estimation.
MC dropout~\citep{gal2016dropout} achieves a similar effect within a single network by treating the variability across multiple stochastic forward passes, induced by dropout at test time, as a proxy for uncertainty.
Deterministic alternatives include evidential deep learning (EDL), which parameterizes a Dirichlet distribution over class probabilities~\citep{sensoy2018evidential}, and spectral-normalized neural Gaussian processes (SNGP), which promote distance-aware uncertainty through spectral normalization and a Gaussian-process output layer~\citep{liu2020simple}.
DivDis instead learns multiple predictive heads under an explicit diversity objective to expose alternative hypotheses~\citep{lee2023diversify}.
Orthogonally, deep mutual learning~\citep{zhang2018deep} trains peer models through reciprocal distillation as a knowledge-transfer technique to improve accuracy.
To the best of our knowledge, its effect on ensemble disagreement has not yet been studied.
Our proposed framework is the first to leverage DML in computational pathology through the lens of uncertainty quantification.

\begin{figure}[t]
    \centering
    \includegraphics[width=\linewidth]{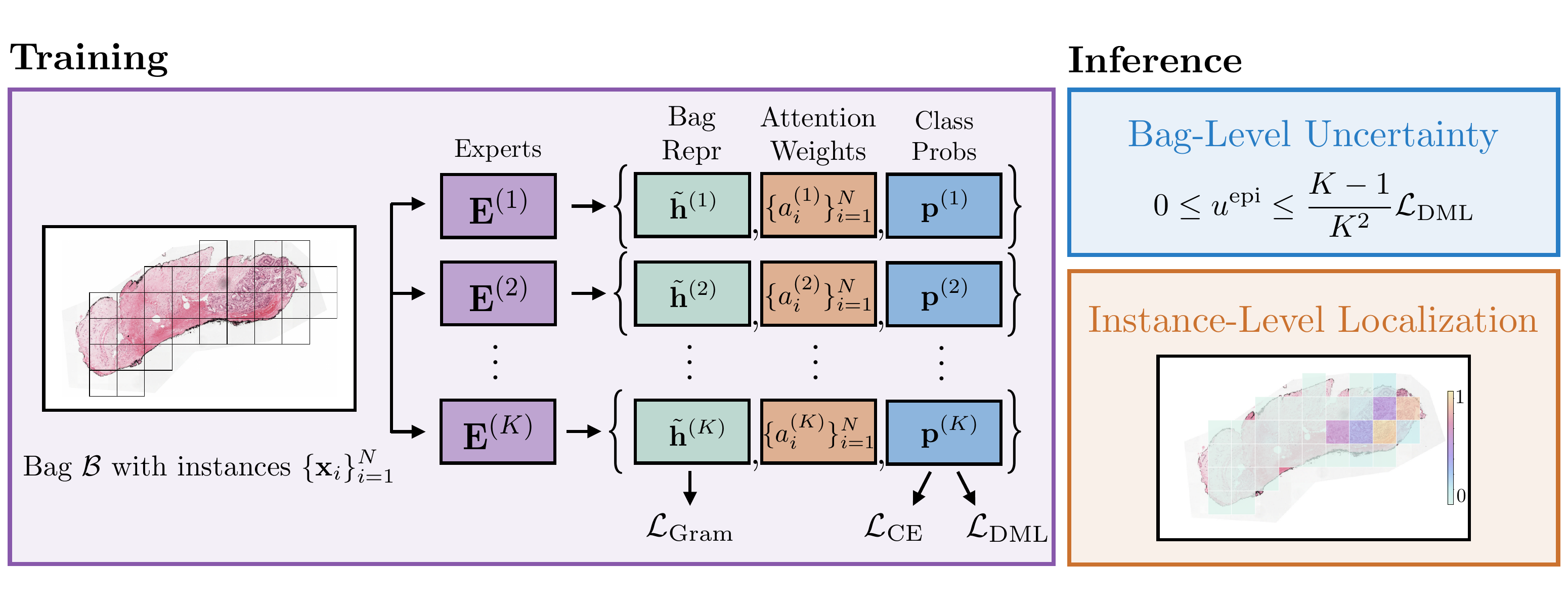}
    \caption{\textbf{Overview of our framework.}
    A whole-slide image (bag) consisting of multiple patches (instances) is processed by $K$ experts, each producing a bag representation, attention weights, and class probabilities.
    Training combines classification, deep mutual learning, and Gramian objectives.
    At inference, the predictive entropy of the expert ensemble, whose epistemic component is theoretically bounded by the DML term, provides slide-level uncertainty scores.
    Agreement in attention maps enables unsupervised patch-level localization.}
    \label{fig:framework}
\end{figure}

\section{Disagreement-informed coordination of experts}

Motivated by the absence of confidence measures and task-dependent performance of PFMs, we introduce \ourmethod (\underline{D}isagreement-\underline{I}nformed \underline{C}oordination of \underline{E}xperts), a plug-and-play framework that leverages multiple PFMs for reliable uncertainty estimation and lesion localization in WSI analysis.

\subsection{Problem formulation}
\label{sec:problem_setup}

We consider a weakly supervised setting in which each whole-slide image is represented as a bag $\mathcal{B} = \{\mathbf{x}_1,\dots,\mathbf{x}_N\}$ of $N$ patches with a slide-level label $y \in \{1,\dots, C\}$, where $N$ varies across slides.
We let $\{\mathrm{PFM}^{(1)},\dots,\mathrm{PFM}^{(K)}\}$ denote $K$ frozen pathology foundation models and emphasize that our formulation is agnostic to the specific selection of models.
Accordingly, each patch $\mathbf{x}_i$ in a slide is encoded as $\mathbf{h}_i^{(k)} = \mathrm{PFM}^{(k)}(\mathbf{x}_i) \in \mathbb{R}^{d_k}$ and the set of embeddings $\{\mathbf{h}_i^{(k)}\}_{i=1}^{N}$ for model $k$ is processed by an independent MIL aggregator $f^{(k)}$.
While our framework applies to any MIL architecture, we instantiate $f^{(k)}$ as ABMIL~\citep{pmlr-v80-ilse18a} throughout, given its widespread use in computational pathology.
We denote a specific PFM-backbone and MIL-head combination as an \textit{expert} $\mathbf{E}^{(k)}$.\footnote{This terminology should not be conflated with mixture-of-experts models, which are outside the scope of this work.} Crucially, our main goal is to learn these expert parameters whose fused prediction achieves competitive performance at the slide level, which is essential for solving the downstream task, while leveraging the predictive entropy of the ensemble as an uncertainty measure.

Thus, each expert computes
\begin{equation*}
  \mathbf{p}^{(k)}(\mathcal{B}) = \mathrm{softmax} \!\left( \mathrm{CLF}^{(k)}\left( \tilde{\mathbf{h}}^{(k)} \right)\!\right),
  \qquad
  \tilde{\mathbf{h}}^{(k)} = \sum_{i=1}^{N} a_{i}^{(k)}\, \mathbf{h}_{i}^{(k)},
\end{equation*}
with attention weights
\begin{equation*}
    a_{i}^{(k)} =
    \frac{\exp\!\Bigl(\mathbf{w}^{(k)\top}
    \bigl[\tanh\!\bigl(\mathbf{V}^{(k)}\mathbf{h}_{i}^{(k)}\bigr)
    \odot
    \sigma\!\bigl(\mathbf{U}^{(k)}\mathbf{h}_{i}^{(k)}\bigr)\bigr]\Bigr)}
    {\sum_{j=1}^{N}\exp\!\Bigl(\mathbf{w}^{(k)\top}
    \bigl[\tanh\!\bigl(\mathbf{V}^{(k)}\mathbf{h}_{j}^{(k)}\bigr)
    \odot
    \sigma\!\bigl(\mathbf{U}^{(k)}\mathbf{h}_{j}^{(k)}\bigr)\bigr]\Bigr)},
\end{equation*}
with $\mathbf{w}^{(k)}$, $\mathbf{V}^{(k)}$, and $\mathbf{U}^{(k)}$ as learnable parameters.
This yields the expert quantities $\{ \tilde{\mathbf{h}}^{(k)}, \{a_i^{(k)}\}_{i=1}^{N}, \mathbf{p}^{(k)}\}_{k=1}^K$, which consist of bag-level embeddings, attention weights for each patch, and a predictive distribution.
The $K$ experts have independent parameters and optimizers, with no weight sharing. 

\subsection{Coupled expert training for prediction and representation alignment}
\label{sec:training}

\paragraph{Prediction-space alignment via deep mutual learning.}
To couple the experts at the level of slide-level decisions, we adopt a deep mutual learning objective~\citep{zhang2018deep}.
Each expert is trained jointly to predict the bag label and to align its predictive distribution with those of its $K-1$ peers.
For a bag $\mathcal{B}$ with label $y$, expert $k$ minimizes
\[
\mathcal{L}_{\mathrm{pred}}^{(k)} = \mathcal{L}_{\mathrm{sup}}^{(k)} + \,\mathcal{L}_{\mathrm{DML}}^{(k)}, \qquad \mathcal{L}_{\mathrm{sup}}^{(k)} = \mathrm{CE}(\mathbf{p}^{(k)}, y),
\]
where $\mathrm{CE}(\cdot)$ is the standard cross-entropy loss and $\mathbf{p}^{(k)}$ is the posterior predictive distribution defined in Section~\ref{sec:problem_setup}.
The mutual distillation term is
\begin{equation}
\mathcal{L}_{\mathrm{DML}}^{(k)} = \frac{1}{K-1} \sum_{\ell\neq k} \mathrm{KL}\!\left( \mathbf{p}^{(\ell)} \,\Big\|\, \mathbf{p}^{(k)} \right),
\label{eq:dml_loss}
\end{equation}
where $\mathrm{KL}(\cdot)$ stands for Kullback–Leibler divergence.
This objective encourages expert $k$ to align its posterior with those of its peers while remaining anchored to the ground-truth label through the supervised loss.

\paragraph{Representation-space alignment via Gramian measure.}
Prediction-level alignment leaves the slide-level embeddings themselves unconstrained.
As a result, experts can produce nearly identical predictions while their bag representations remain geometrically misaligned.
Therefore, we add a representation-level alignment term based on the
Gramian measure~\citep{cicchetti2025gramian}, which encourages slide embeddings to span a low-volume subspace.
Concretely, we let $\psi^{(k)}$ be a projection head that maps $\tilde{\mathbf{h}}^{(k)}$ to a common dimensionality, before normalizing the resulting embedding,
\[
\mathbf{r}^{(k)} = \psi^{(k)}\!\left(\tilde{\mathbf{h}}^{(k)}\right) \in \mathbb{R}^{d_r},
\qquad
\tilde{\mathbf{r}}^{(k)} = \mathbf{r}^{(k)} / \|\mathbf{r}^{(k)}\|_2 .
\]
Stacking the normalized slide-level embeddings column-wise yields $\tilde{\mathbf{R}}(\mathcal{B}) = [\tilde{\mathbf{r}}^{(1)},\dots,\tilde{\mathbf{r}}^{(K)}] \in \mathbb{R}^{d_r \times K}$, and the Gramian loss is defined as the volume of the parallelotope they span, i.e., 
\[
\mathcal{L}_{\mathrm{Gram}} =
\sqrt{\det\!\left(\tilde{\mathbf{R}}^{\top}\tilde{\mathbf{R}}\right)}.
\]
A small value of $\mathcal{L}_{\mathrm{Gram}}$ indicates that the $K$ slide embeddings lie close to a shared low-dimensional subspace, whereas a large value reflects semantic divergence across experts.

\begin{algorithm}[t]
\small
\caption{Training of the proposed framework on a bag $\mathcal{B} = \{\mathbf{x}_i\}_{i=1}^{N}$ with label $y$.}
\label{alg:framework}
\begin{algorithmic}
\Require Frozen PFMs $\{\mathrm{PFM}^{(k)}\}$, MIL heads $\{f^{(k)}\}$ with parameters $\{\theta_k\}$, projection heads $\{\psi^{(k)}\}$, per-expert optimizers $\{\mathrm{Opt}_k\}$, and $\lambda_{\mathrm{Gram}}(t)$.
\For{$k=1,\dots,K$} \Comment{forward pass for each expert}
   \State $\mathbf{h}_i^{(k)} \gets \mathrm{PFM}^{(k)}(\mathbf{x}_i)$ for $i=1,\dots,N$
   \State $\bigl(\tilde{\mathbf{h}}^{(k)},\,\{a_i^{(k)}\}_{i=1}^{N},\,\mathbf{p}^{(k)}\bigr) \gets f^{(k)}\!\bigl(\{\mathbf{h}_i^{(k)}\}_{i=1}^{N}\bigr)$
   \State $\mathbf{r}^{(k)} \gets \psi^{(k)}(\tilde{\mathbf{h}}^{(k)})$,\quad $\tilde{\mathbf{r}}^{(k)} \gets \mathbf{r}^{(k)}/\|\mathbf{r}^{(k)}\|_2$
\EndFor
\State $\tilde{\mathbf{R}} \gets [\tilde{\mathbf{r}}^{(1)},\dots,\tilde{\mathbf{r}}^{(K)}]$ \Comment{shared regularization across experts}
\State $\mathcal{L}_{\mathrm{Gram}} \gets \sqrt{\det(\tilde{\mathbf{R}}^{\top}\tilde{\mathbf{R}})}$
\For{$k=1,\dots,K$} \Comment{per-expert loss and update}
   \State $\mathcal{L}_{\mathrm{sup}}^{(k)} \gets \mathrm{CE}(\mathbf{p}^{(k)}, y)$
   \State $\mathcal{L}_{\mathrm{DML}}^{(k)} \gets \dfrac{1}{K-1}\sum\limits_{\ell\neq k}\mathrm{KL}\bigl(\mathbf{p}^{(\ell)}\,\|\,\mathbf{p}^{(k)}\bigr)$
   \State $\mathcal{L}^{(k)} \gets \mathcal{L}_{\mathrm{sup}}^{(k)} + \,\mathcal{L}_{\mathrm{DML}}^{(k)} + \lambda_{\mathrm{Gram}}(t)\,\mathcal{L}_{\mathrm{Gram}}$
   \State $\theta_k \gets \mathrm{Opt}_k\!\bigl(\theta_k,\,\nabla_{\theta_k}\mathcal{L}^{(k)}\bigr)$
\EndFor
\end{algorithmic}
\end{algorithm}

\paragraph{Training objective.}
Finally, the total loss on bag $\mathcal{B}$ combines the three terms introduced above.
\begin{equation}
    \mathcal{L} =
\sum\limits_{k=1}^{K} \mathcal{L}_{\mathrm{sup}}^{(k)}
+ \sum\limits_{k=1}^{K} \mathcal{L}_{\mathrm{DML}}^{(k)}
+ \lambda_{\mathrm{Gram}}(t)\, \mathcal{L}_{\mathrm{Gram}} = \mathcal{L}_{\mathrm{sup}}
+ \mathcal{L}_{\mathrm{DML}}
+ \lambda_{\mathrm{Gram}}(t)\, \mathcal{L}_{\mathrm{Gram}}
\label{eq:total_loss}
\end{equation}
where $\lambda_{\mathrm{Gram}}(t)$ follows a warm-up schedule over the first few epochs.
We summarize the training procedure in Algorithm~\ref{alg:framework}  and illustrate the full framework in Figure~\ref{fig:framework}.

\subsection{Uncertainty representation}
\label{sec:uq-repre}

At inference time, we obtain the ensemble prediction by averaging the predictive distributions of the $K$ experts in \ourmethod as $\bar{\mathbf{p}}(\mathcal{B}) \;=\; \frac{1}{K}\sum_{k=1}^{K} \mathbf{p}^{(k)}(\mathcal{B})$.
Following the standard information-theoretic decomposition~\citep{depeweg2018decomposition, lakshminarayanan2017simple}, we then disentangle this prediction into its predictive (total), aleatoric (data), and epistemic (model) components:
\[
\underbrace{u^{\mathrm{pred}} = \mathcal{H}\!\left(\bar{\mathbf{p}}(\mathcal{B})\right)\vphantom{\sum_{k=1}^{K}}}_{\text{Total Uncertainty}}
\;,\quad
\underbrace{u^{\mathrm{alea}} = \frac{1}{K}\sum_{k=1}^{K}\mathcal{H}\!\left(\mathbf{p}^{(k)}(\mathcal{B})\right)}_{\text{Data Uncertainty}}
\;,\quad
\underbrace{u^{\mathrm{epi}} = u^{\mathrm{pred}} - u^{\mathrm{alea}}\vphantom{\sum_{k=1}^{K}}}_{\text{Model Uncertainty}}
\]
where $\mathcal{H}(\cdot)$ is the Shannon entropy.
Intuitively, the predictive term $u^{\mathrm{pred}}$ is high whenever $\bar{\mathbf{p}}(\mathcal{B})$ is spread across classes, meaning the ensemble is uncertain about which class the sample belongs to.
The aleatoric term $u^{\mathrm{alea}}$, is the irreducible uncertainty arising from inherent noise in the data generation process, whereas the epistemic term $u^{\mathrm{epi}}$, is the reducible component of the total uncertainty.
It captures the residual disagreement across experts and vanishes whenever all $K$ experts predict the same distribution.
Although this decomposition is often motivated by a Bayesian view of the ensemble, its information-theoretic form is valid without that assumption, which we show below.

\begin{proposition}
\label{prop:jsd}
For any $K$ distributions $\mathbf{p}^{(1)},\dots,\mathbf{p}^{(K)}$ over $C$ classes
\[
u^{\mathrm{epi}}
= \frac{1}{K}\sum_{k=1}^{K}\mathrm{KL}\!\left(\mathbf{p}^{(k)}\,\|\,\bar{\mathbf{p}}\right)
= \mathrm{JSD}\!\left(\mathbf{p}^{(1)},\dots,\mathbf{p}^{(K)}\right) \;\ge\; 0,
\]
with equality to zero if and only if $\mathbf{p}^{(1)}=\dots=\mathbf{p}^{(K)}$.
\end{proposition}

Proposition~\ref{prop:jsd} shows that $u^{\mathrm{epi}}$ admits a closed-form expression as the multi-way Jensen–Shannon divergence among expert posteriors~\citep{lin1991divergence} and is lower-bounded by zero.
We now complement this result by showing that $u^{\mathrm{epi}}$ also admits an upper bound.

\begin{proposition}
\label{prop:dml_bound}
For a bag $\mathcal{B}$, the epistemic uncertainty satisfies
\[
u^{\mathrm{epi}} = \frac{1}{K}\sum_{k=1}^{K} \mathrm{KL}\!\left( \mathbf{p}^{(k)} \,\Big\|\, \bar{\mathbf{p}} \right) \;\le\; \frac{K-1}{K^2}\mathcal{L}_{\mathrm{DML}}.
\]
\end{proposition}

Proposition~\ref{prop:dml_bound} establishes a formal connection between deep mutual learning and uncertainty estimation and shows that minimizing the aggregate DML objective in Equation~\ref{eq:dml_loss} contracts an upper bound on the epistemic uncertainty of the ensemble.
Concretely, as the DML loss decreases during training, inter-expert disagreement shrinks, which in turn reduces epistemic uncertainty.
Thus, DML not only facilitates information exchange between experts to improve predictive performance, but also promotes more consistent posterior beliefs.
We provide proofs in Appendix~\ref{app:proofs}.
\section{Experiments}

\subsection{Setup}
\label{sec:setup}

\paragraph{Datasets.}
We evaluate on three publicly available digital pathology datasets that span two cancer types and both binary and multi-class settings: PANDA~\citep{bulten2022artificial} for six-class ISUP grading of prostate biopsies, and CAMELYON16 (C16)~\citep{ehteshami2017diagnostic} and CAMELYON17 (C17)~\citep{litjens20181399} for binary metastasis detection in breast lymph-node sections.
All provide pixel-level lesion annotations, enabling our patch-level localization experiments.
Slides are tiled into non-overlapping $512\times512$-pixel patches at $20\times$ magnification using TRIDENT~\citep{zhang2025standardizing}.
When no official test split is available, we hold out 20\% of the slides as a test set and perform five-fold stratified cross-validation on the remainder to tune hyperparameters via Optuna~\citep{optuna_2019}.
Full dataset statistics, splits, and implementation details are in Appendix~\ref{app:datasets} and~\ref{app:implementation}.

\paragraph{Training of \ourmethod.}
Following ensemble learning conventions~\citep{lakshminarayanan2017simple}, we instantiate $K=5$ experts using Virchow2~\citep{zimmermann2024virchow2}, UNI2-h~\citep{chen2024uni}, H-optimus-1~\citep{hoptimus1}, CONCHv1.5~\citep{lu2024visual}, and Hibou-L~\citep{nechaev2024hibou}, spanning diverse pretraining strategies (DINOv2, CoCa), model sizes ($307$M to $1.1$B parameters), and data scales ($200$M to $1.7$B patches).
Each PFM is kept frozen and coupled with an independent MIL head.
Unless otherwise stated, we use ABMIL~\citep{pmlr-v80-ilse18a} as the default aggregator due to its widespread use.
The $K$ heads are optimized jointly under the objective in Equation~\ref{eq:total_loss} with one AdamW~\citep{loshchilov2018decoupled} optimizer per expert.
The Gramian weight $\lambda_{\mathrm{Gram}}(t)$ follows a warm-up schedule, allowing experts to first specialize under their supervised and DML objectives, which we found empirically to prevent trivial solutions and representation collapse.
We evaluate two variants: \ourmethodnoreg, trained under the supervised loss and the prediction-space DML objective only, and \ourmethod, which additionally applies the Gramian objective.
Full model and training details are in Appendices~\ref{app:baselines} and~\ref{app:implementation} along with an analysis of inference computational cost in Appendix~\ref{app:inference_cost}.

\paragraph{Baselines.}
We compare \ourmethod against three method families, all adapted to the same MIL setting.
\textit{Single PFMs} pair each frozen PFM with its own ABMIL head trained independently.
\textit{UQ Methods} are common uncertainty quantification techniques in the literature, spanning Monte Carlo dropout~\citep{gal2016dropout}, deep ensembles~\citep{lakshminarayanan2017simple}, evidential deep learning~\citep{sensoy2018evidential}, SNGP~\citep{liu2020simple}, and DivDis~\citep{lee2023diversify}.
\textit{Fusion Methods} consist of early fusion, which feeds concatenated patch embeddings from all $K$ PFMs to one ABMIL head, and late fusion, which averages posteriors from independently trained heads on the $K$ distinct PFMs.
Implementation details are given in Appendix~\ref{app:baselines}.
The concurrent fusion methods of FuseCPath~\citep{yang2025fusion}, ELF~\citep{luo2025ensemble}, and FM2~\citep{yu2025fm2} either provide no public code and pretrained weights or release only partial implementations.
PICTURE~\citep{zhao2025uncertainty} is, by design, specific to glioblastoma diagnosis and does not transfer to the prostate and breast tasks studied here without dedicated clinical curation.
See Appendix~\ref{app:concurrent} for a detailed discussion.

\paragraph{Metrics.}
We assess our framework along three axes: slide-level classification via F1 score, calibration via negative log-likelihood (NLL) of the slide-level posteriors, and patch-level lesion localization via the AUC of mean attention scores against binary tumor labels.
Additional metrics along each axis (AUC, Brier score) and their formal definitions are reported in Appendix~\ref{app:metrics}.

\subsection{Results}

\subsubsection{Classification and calibration performance}

Before turning to uncertainty estimation, we first assess the models' classification and calibration performance.
This step allows us to evaluate whether the models can solve the task and whether their predictive distributions provide a meaningful basis for the subsequent uncertainty analyses.

\begin{table}[t]
\caption{\textbf{Slide- and patch-level performance on PANDA, C16, and C17.} We report slide-level classification with F1 (\%), calibration with NLL ($\times 100$), and patch-level tumor localization with AUC (\%). Values are mean $\pm$ std across five folds on the held-out test set. MCD and DE are built on the single PFM with the highest test F1 per dataset, while fusion and \ourmethod{} variants use all five PFMs. Bold and underline indicate the best and second-best per column. Two-sided paired $t$-tests across the five folds compare \ourmethod{} with each baseline. Significant improvements ($p<0.05$) for \ourmethod{} and the baseline are marked with $^{\dagger}$ and $^{\ddagger}$, respectively. Tables with additional metrics are in Appendix~\ref{app:results}.}
\label{tab:abmil-main-performance}
\centering
\resizebox{\textwidth}{!}{%
\begin{tabular}{lccccccccc}
\toprule
 & \multicolumn{3}{c}{\textbf{PANDA}} & \multicolumn{3}{c}{\textbf{CAMELYON16}} & \multicolumn{3}{c}{\textbf{CAMELYON17}} \\
\cmidrule(lr){2-4}\cmidrule(lr){5-7}\cmidrule(lr){8-10}
 & \multicolumn{1}{c}{\textbf{Classification}} & \multicolumn{1}{c}{\textbf{Calibration}} & \multicolumn{1}{c}{\textbf{Localization}} & \multicolumn{1}{c}{\textbf{Classification}} & \multicolumn{1}{c}{\textbf{Calibration}} & \multicolumn{1}{c}{\textbf{Localization}} & \multicolumn{1}{c}{\textbf{Classification}} & \multicolumn{1}{c}{\textbf{Calibration}} & \multicolumn{1}{c}{\textbf{Localization}} \\
 & F1 ($\uparrow$) & NLL ($\downarrow$) & AUC ($\uparrow$) & F1 ($\uparrow$) & NLL ($\downarrow$) & AUC ($\uparrow$) & F1 ($\uparrow$) & NLL ($\downarrow$) & AUC ($\uparrow$) \\
\midrule
\multicolumn{10}{l}{\textbf{Single PFMs}} \\
Virchow2 & 72.8 \scriptsize{$\pm$ 1.51}$^{\dagger}$ & 58.0 \scriptsize{$\pm$ 1.30}$^{\dagger}$ & \underline{82.0 \scriptsize{$\pm$ 0.67}}$^{\ddagger}$ & \underline{96.3 \scriptsize{$\pm$ 0.59}} & 14.8 \scriptsize{$\pm$ 2.44} & 85.5 \scriptsize{$\pm$ 5.26} & 79.8 \scriptsize{$\pm$ 2.24} & 58.6 \scriptsize{$\pm$ 11.55} & 96.0 \scriptsize{$\pm$ 2.73} \\
UNI2-h & 69.7 \scriptsize{$\pm$ 0.82}$^{\dagger}$ & 60.0 \scriptsize{$\pm$ 0.55}$^{\dagger}$ & 73.9 \scriptsize{$\pm$ 1.11}$^{\dagger}$ & 95.7 \scriptsize{$\pm$ 1.96} & 20.6 \scriptsize{$\pm$ 8.26} & 86.9 \scriptsize{$\pm$ 1.91} & \textbf{80.6 \scriptsize{$\pm$ 3.47}} & 60.2 \scriptsize{$\pm$ 10.49} & 86.4 \scriptsize{$\pm$ 7.08}$^{\dagger}$ \\
H-optimus-1 & 71.2 \scriptsize{$\pm$ 0.87}$^{\dagger}$ & 59.1 \scriptsize{$\pm$ 0.92}$^{\dagger}$ & 74.2 \scriptsize{$\pm$ 2.06}$^{\dagger}$ & 94.4 \scriptsize{$\pm$ 1.15}$^{\dagger}$ & 18.8 \scriptsize{$\pm$ 3.38}$^{\dagger}$ & 90.5 \scriptsize{$\pm$ 2.11} & 78.2 \scriptsize{$\pm$ 2.27} & 65.4 \scriptsize{$\pm$ 12.87} & 95.1 \scriptsize{$\pm$ 1.18} \\
CONCHv1.5 & 68.7 \scriptsize{$\pm$ 1.32}$^{\dagger}$ & 66.3 \scriptsize{$\pm$ 0.80}$^{\dagger}$ & 78.9 \scriptsize{$\pm$ 1.15}$^{\dagger}$ & 94.5 \scriptsize{$\pm$ 0.99}$^{\dagger}$ & 16.9 \scriptsize{$\pm$ 3.51}$^{\dagger}$ & 84.3 \scriptsize{$\pm$ 3.01}$^{\dagger}$ & 78.0 \scriptsize{$\pm$ 3.48} & 59.0 \scriptsize{$\pm$ 11.16} & \textbf{96.9 \scriptsize{$\pm$ 1.11}} \\
Hibou-L & 69.9 \scriptsize{$\pm$ 0.88}$^{\dagger}$ & 62.6 \scriptsize{$\pm$ 0.72}$^{\dagger}$ & 73.2 \scriptsize{$\pm$ 1.06}$^{\dagger}$ & 92.4 \scriptsize{$\pm$ 0.90}$^{\dagger}$ & 26.7 \scriptsize{$\pm$ 4.80}$^{\dagger}$ & 81.4 \scriptsize{$\pm$ 12.06} & 79.2 \scriptsize{$\pm$ 3.86} & 65.9 \scriptsize{$\pm$ 7.82}$^{\dagger}$ & 95.5 \scriptsize{$\pm$ 1.51} \\
\midrule
\multicolumn{10}{l}{\textbf{UQ Methods}} \\
MC Dropout & 72.3 \scriptsize{$\pm$ 1.69}$^{\dagger}$ & 56.3 \scriptsize{$\pm$ 1.70}$^{\dagger}$ & 78.5 \scriptsize{$\pm$ 1.30}$^{\dagger}$ & 95.2 \scriptsize{$\pm$ 0.60}$^{\dagger}$ & 15.9 \scriptsize{$\pm$ 5.22} & 81.5 \scriptsize{$\pm$ 9.09} & 79.0 \scriptsize{$\pm$ 5.61} & 68.6 \scriptsize{$\pm$ 29.89} & 90.1 \scriptsize{$\pm$ 5.03}$^{\dagger}$ \\
Deep Ensemble & 73.8 \scriptsize{$\pm$ 0.11}$^{\dagger}$ & 53.8 \scriptsize{$\pm$ 0.95}$^{\dagger}$ & \textbf{82.8 \scriptsize{$\pm$ 0.87}}$^{\ddagger}$ & 95.8 \scriptsize{$\pm$ 0.00} & 12.1 \scriptsize{$\pm$ 1.72} & 87.0 \scriptsize{$\pm$ 2.18} & 79.5 \scriptsize{$\pm$ 2.38}$^{\ddagger}$ & 60.3 \scriptsize{$\pm$ 11.22} & 91.7 \scriptsize{$\pm$ 3.59} \\
EDL & 61.4 \scriptsize{$\pm$ 2.35}$^{\dagger}$ & 83.6 \scriptsize{$\pm$ 1.23}$^{\dagger}$ & 71.7 \scriptsize{$\pm$ 2.92}$^{\dagger}$ & 95.8 \scriptsize{$\pm$ 1.00} & 16.1 \scriptsize{$\pm$ 3.77}$^{\dagger}$ & 88.0 \scriptsize{$\pm$ 8.26} & 80.2 \scriptsize{$\pm$ 1.11}$^{\ddagger}$ & 69.5 \scriptsize{$\pm$ 13.98} & 84.5 \scriptsize{$\pm$ 7.06}$^{\dagger}$ \\
SNGP & \underline{74.4 \scriptsize{$\pm$ 1.14}}$^{\dagger}$ & 56.9 \scriptsize{$\pm$ 1.41}$^{\dagger}$ & 81.4 \scriptsize{$\pm$ 0.59} & 96.0 \scriptsize{$\pm$ 0.49} & 16.8 \scriptsize{$\pm$ 1.10}$^{\dagger}$ & 56.6 \scriptsize{$\pm$ 2.28}$^{\dagger}$ & 69.1 \scriptsize{$\pm$ 1.65}$^{\dagger}$ & 79.6 \scriptsize{$\pm$ 8.14}$^{\dagger}$ & 72.6 \scriptsize{$\pm$ 4.75}$^{\dagger}$ \\
DivDis & 71.6 \scriptsize{$\pm$ 1.24}$^{\dagger}$ & 59.7 \scriptsize{$\pm$ 1.17}$^{\dagger}$ & 81.1 \scriptsize{$\pm$ 2.36} & 95.9 \scriptsize{$\pm$ 0.67} & 14.4 \scriptsize{$\pm$ 3.86} & 89.4 \scriptsize{$\pm$ 1.57} & 79.0 \scriptsize{$\pm$ 5.92} & 73.1 \scriptsize{$\pm$ 31.20} & 92.7 \scriptsize{$\pm$ 1.71}$^{\dagger}$ \\
\midrule
\multicolumn{10}{l}{\textbf{Fusion Methods}} \\
Early Fusion & 73.6 \scriptsize{$\pm$ 1.02}$^{\dagger}$ & 54.5 \scriptsize{$\pm$ 0.71}$^{\dagger}$ & 80.3 \scriptsize{$\pm$ 1.22} & 95.7 \scriptsize{$\pm$ 1.75} & 19.6 \scriptsize{$\pm$ 4.64}$^{\dagger}$ & \textbf{93.3 \scriptsize{$\pm$ 0.87}}$^{\ddagger}$ & \underline{80.3 \scriptsize{$\pm$ 3.54}} & 57.5 \scriptsize{$\pm$ 10.67} & 95.8 \scriptsize{$\pm$ 1.75} \\
Late Fusion & 73.4 \scriptsize{$\pm$ 0.81}$^{\dagger}$ & \underline{53.3 \scriptsize{$\pm$ 1.20}}$^{\dagger}$ & 77.8 \scriptsize{$\pm$ 0.72}$^{\dagger}$ & 96.2 \scriptsize{$\pm$ 0.57} & \underline{11.7 \scriptsize{$\pm$ 2.56}} & 92.5 \scriptsize{$\pm$ 1.09}$^{\ddagger}$ & 78.9 \scriptsize{$\pm$ 2.99} & \underline{54.2 \scriptsize{$\pm$ 8.65}} & \textbf{96.9 \scriptsize{$\pm$ 1.05}} \\
\midrule
\multicolumn{10}{l}{\textbf{Ours}} \\
\rowcolor{gray!20} \ourmethodnoreg & 73.6 \scriptsize{$\pm$ 1.11} & 54.6 \scriptsize{$\pm$ 2.97} & 80.1 \scriptsize{$\pm$ 1.25} & 96.2 \scriptsize{$\pm$ 0.57} & 12.2 \scriptsize{$\pm$ 1.06} & \underline{93.0 \scriptsize{$\pm$ 1.51}} & 79.3 \scriptsize{$\pm$ 1.46} & \textbf{50.5 \scriptsize{$\pm$ 2.98}} & 93.2 \scriptsize{$\pm$ 2.45} \\
\rowcolor{gray!20} \ourmethod & \textbf{76.6 \scriptsize{$\pm$ 1.02}} & \textbf{48.9 \scriptsize{$\pm$ 1.68}} & 81.1 \scriptsize{$\pm$ 0.62} & \textbf{96.4 \scriptsize{$\pm$ 0.58}} & \textbf{11.5 \scriptsize{$\pm$ 0.91}} & 89.4 \scriptsize{$\pm$ 1.45} & 77.2 \scriptsize{$\pm$ 1.36} & 54.6 \scriptsize{$\pm$ 3.04} & \underline{96.2 \scriptsize{$\pm$ 0.60}} \\
\bottomrule
\end{tabular}}
\end{table}

\paragraph{Single PFM rankings vary across datasets, tasks, and metrics.}
The single PFM results reinforce the motivation for our fusion framework.
Virchow2 and UNI2-h are the strongest classifiers, yet neither is consistently best on calibration or localization, where CONCHv1.5 and H-optimus-1 each take the top single PFM slot on at least one cohort.
Thus, single PFM selection is task-dependent.

\paragraph{Fusion generally improves slide-level performance.}
On PANDA, \ourmethod delivers the largest gains, lifting F1 by $3.8$ percentage points over the best single PFM and by $3.0$ over the best fusion method, while reducing NLL by similar margins.
These gains are statistically significant across folds.
On C16, where single PFMs are already near saturation, \ourmethod still tops both classification and calibration, but the margins over the strongest baselines are within fold variation.
On C17, prediction-space alignment alone is preferable, and adding the Gramian term slightly decreases F1.
Across all three datasets, fusion methods generally outperform single PFMs on slide-level metrics, with classification and calibration gains often appearing jointly.

\subsubsection{Slide-level uncertainty estimation}

We evaluate whether the predictive entropy of the ensemble posterior can identify likely errors in practice.
To probe its reliability, we conduct three complementary analyses.
First, we implement a threshold-free selective-prediction experiment that defers slides in decreasing uncertainty and measures retained error rates.
Second, we conduct a deployment-oriented experiment by tuning the deferral threshold on validation folds and applying it unchanged to the test set.
Both analyses are reported on PANDA and C17, since C16 is near-saturated (test error below $3\%$ for most methods), leaving too few misclassified slides for a statistically meaningful separation analysis.
Finally, we conduct a cohort-shift experiment that tests whether predictive uncertainty separates correct from incorrect predictions when models trained on one CAMELYON cohort are evaluated on the other.
The resulting distribution shift induces enough additional errors to make this analysis informative in both directions.
All three analyses are restricted to UQ methods.
MC dropout and deep ensemble are built on Virchow2 for PANDA and C16, whereas UNI2-h is used for C17.

\begin{figure}[t]
    \centering
    \begin{subfigure}[b]{0.29\linewidth}
        \raisebox{6pt}{\includegraphics[width=\linewidth]{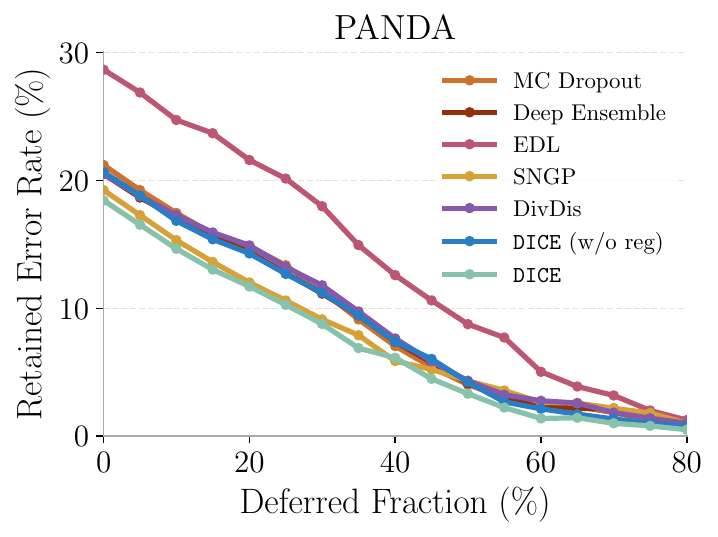}}
    \end{subfigure}
    \hfill
    \begin{subfigure}[b]{0.70\linewidth}
       \includegraphics[width=\linewidth]{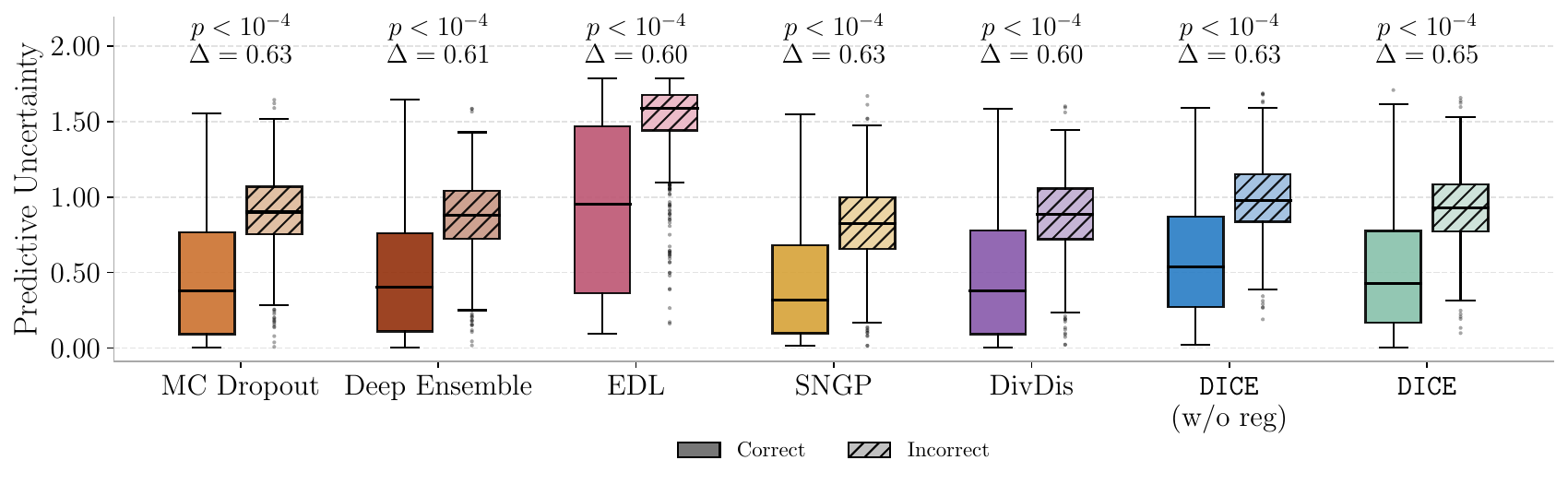}
    \end{subfigure}
   \caption{\textbf{Slide-level predictive uncertainty on PANDA.}
   \textit{Left:} We defer test slides in decreasing predictive uncertainty and report error rate on the retained ones.
   Lower curves indicate that low-uncertainty slides contain a smaller fraction of errors.
   \textit{Right:} Predictive uncertainty distributions for correct and incorrect predictions, shown for UQ baselines and \ourmethod variants.
   Annotations report the Mann-Whitney $U$ test $p$-value and Cliff's $\Delta$ effect size.
   A better separation (larger $|\Delta|$) indicates the method assigns higher uncertainty to misclassified slides, with \ourmethod achieving the best performance.}
    \label{fig:slide_uncertainty_panda}
\end{figure}

\begin{figure}[t]
    \centering
    \begin{subfigure}[b]{0.29\linewidth}
        \raisebox{6pt}{\includegraphics[width=\linewidth]{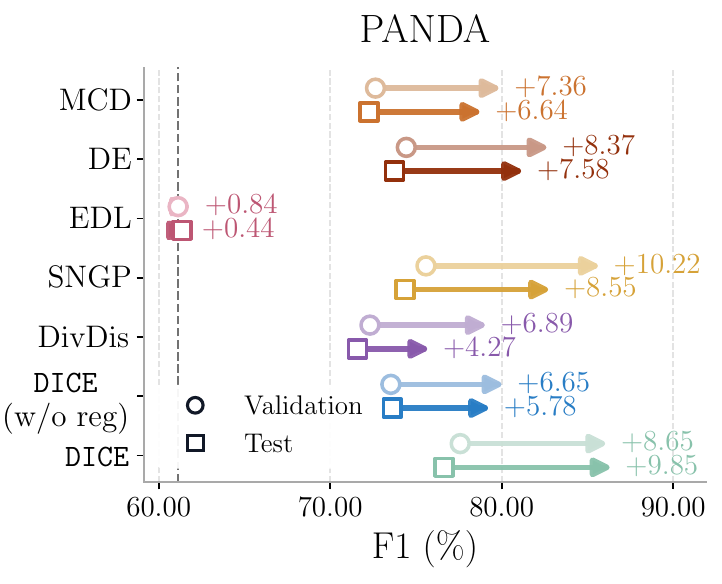}}
        \caption{Classification with rejection.}
        \label{fig:threshold_rejection_panda}
    \end{subfigure}
    \hfill
    \begin{subfigure}[b]{0.70\linewidth}
       \includegraphics[width=\linewidth]{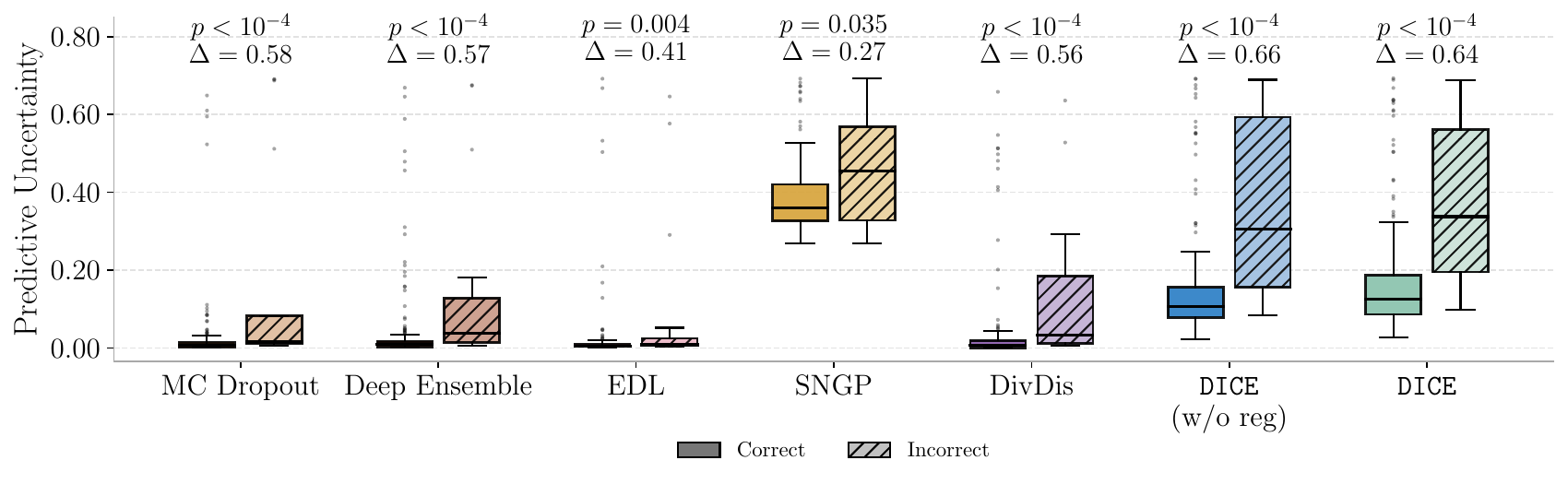}
            \caption{Cross-cohort generalization.}
        \label{fig:16to17}
    \end{subfigure}
    \caption{\textbf{\ourmethod's uncertainty signal generalizes across data splits and cohorts.}
    \textit{Left:} F1 (\%) before vs. after rejecting slides whose predictive uncertainty exceeds a validation-tuned threshold, on validation (light) and test (dark).
    \textit{Right:} Predictive uncertainty distributions for correct and incorrect predictions on C17 from models trained on C16. Annotations as in Figure~\ref{fig:slide_uncertainty_panda}.}
    \label{fig:generalization}
\end{figure}

\paragraph{Predictive uncertainty identifies failure-prone slides.}
Figure~\ref{fig:slide_uncertainty_panda} shows the PANDA results for the threshold-free selective prediction analysis, while the analogue for C17 is given in Appendix~\ref{app:results}.
We observe that deferring high-entropy slides monotonically lowers the retained error rate.
Moreover, \ourmethod dominates the UQ baselines across the full deferral range, starting from the lowest error at zero deferral.
The box plots reveal that this gain comes from improved separation.
\ourmethod shifts the median entropy of correctly classified slides downward while keeping the entropy of misclassified slides high, yielding a Cliff's $\Delta = 0.65$ ($p < 10^{-4}$), which is the largest among all methods.

\paragraph{Validation-tuned rejection transfers to held-out test data.}
Figure~\ref{fig:threshold_rejection_panda} shows the change in F1 before and after rejecting slides with predictive entropy exceeding the selected threshold.
On PANDA, \ourmethod attains the highest post-rejection F1 on test, and its test-time gain exceeds its validation gain.
On C17, whose plot is given in Figure~\ref{fig:threshold_rejection_all} of Appendix~\ref{app:results}, the validation-tuned thresholds transfer to test only for \ourmethod and \ourmethodnoreg, suggesting the generalizability of our uncertainty signal.

\paragraph{Predictive entropy remains informative under cohort shift.}
Figure~\ref{fig:16to17} shows the correct/incorrect entropy distributions on C17 from models trained on C16, and the reverse direction is reported in Appendix~\ref{app:results}.
\ourmethodnoreg achieves the strongest separation with $\Delta = 0.66$, outperforming MC dropout and deep ensembles.
This shows that under a cohort shift, principled single-model uncertainty may not be sufficient, while the ensemble one continues to detect likely errors.

\begin{figure}[t]
    \centering
    \includegraphics[width=\linewidth]{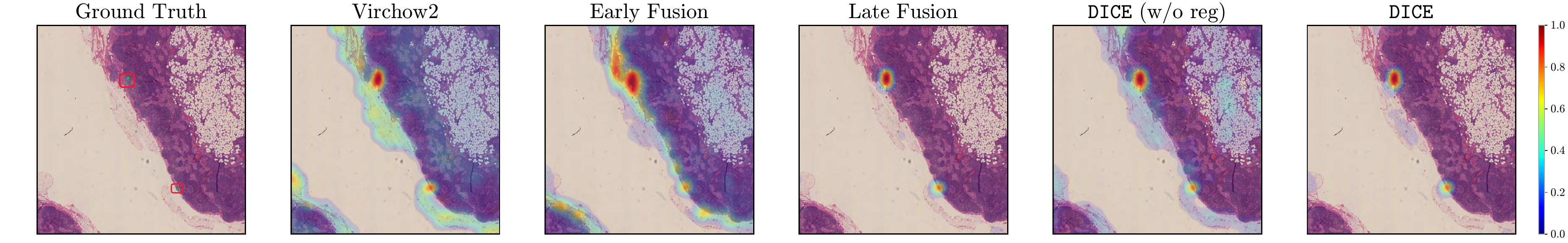}
    \caption{\textbf{Patch-level lesion localization on C16.} From left to right: a zoomed-in segment of a WSI with its ground-truth tumor annotations, heatmaps of attention scores of the single PFM with the highest test F1 and of fusion baselines, \ourmethodnoreg, and \ourmethod. Best viewed in color.}
    \label{fig:attention_score_camelyon16}
\end{figure}

\subsubsection{Patch-level lesion localization}
\label{sec:localization}

Recent work shows that different PFMs attend to different tissue regions when making the same prediction~\citep{neidlinger2025benchmarking}, which motivates assessing whether expert agreement yields more reliable localization than any single model.
We therefore carry out a supplementary analysis to understand whether \ourmethod bases its decisions on clinically relevant areas.
To this end, we define the patch-level consensus as the empirical mean of the $K$ expert attention coefficients, $c^{\mathrm{patch}}_i = \frac{1}{K}\sum_{k=1}^{K}a_i^{(k)}$, which is high on patches that all experts attend to and low on those they consistently ignore.
For late fusion and the two \ourmethod variants we use $c^{\mathrm{patch}}_i$ as the localization signal, while for the single PFM baselines and early fusion we use the MIL head's attention coefficients directly.

\paragraph{Patch-level localization remains competitive despite no supervision.}
As seen in Table~\ref{tab:abmil-main-performance}, on C16, early fusion attains the highest localization AUC overall, with \ourmethodnoreg a close second, and both exceed every single PFM by at least $2.5$ percentage points.
On C17, \ourmethod matches the strongest single PFM within fold-level variability.
This experiment shows that attention consensus is a sharp localization signal, especially valuable for the small lesions that are easiest to miss and most time-consuming to localize by hand.
Figure~\ref{fig:attention_score_camelyon16} illustrates this on a C16 test slide.
The single PFM attention spreads across stromal tissue near the lesion boundary, while \ourmethod concentrates attention inside the annotated tumor regions.
Similar visualizations for all datasets are in Appendix~\ref{app:results}. 

\begin{table}[h]
\caption{\textbf{Ablation of the MIL head.} We report slide-level AUC (\%) with mean $\pm$ std across five folds in the held-out test set. ABMIL columns report absolute values. DSMIL and TransMIL report absolute values with percentage-point deltas relative to ABMIL.}
\label{tab:mil-head-5encoder-auc-delta}
\centering
\resizebox{\textwidth}{!}{%
\begin{tabular}{lccccccccc}
\toprule
 & \multicolumn{3}{c}{\textbf{PANDA}} & \multicolumn{3}{c}{\textbf{CAMELYON16}} & \multicolumn{3}{c}{\textbf{CAMELYON17}} \\
\cmidrule(lr){2-4}\cmidrule(lr){5-7}\cmidrule(lr){8-10}
 & ABMIL & DSMIL ($\Delta$) & TransMIL ($\Delta$) & ABMIL & DSMIL ($\Delta$) & TransMIL ($\Delta$) & ABMIL & DSMIL ($\Delta$) & TransMIL ($\Delta$) \\
\midrule
\ourmethodnoreg & 96.1 \scriptsize{$\pm$ 0.27} & 96.1 \scriptsize{$\pm$ 0.28} (+0.1) & 95.6 \scriptsize{$\pm$ 0.07} (-0.5) & 98.5 \scriptsize{$\pm$ 0.33} & 98.9 \scriptsize{$\pm$ 0.37} (+0.4) & 98.9 \scriptsize{$\pm$ 0.39} (+0.4) & 92.4 \scriptsize{$\pm$ 0.70} & 91.9 \scriptsize{$\pm$ 1.55} (-0.6) & 91.4 \scriptsize{$\pm$ 0.83} (-1.0) \\
\ourmethod & 96.7 \scriptsize{$\pm$ 0.20} & 95.9 \scriptsize{$\pm$ 0.24} (-0.8) & 94.8 \scriptsize{$\pm$ 0.75} (-1.9) & 98.7 \scriptsize{$\pm$ 0.48} & 98.6 \scriptsize{$\pm$ 0.42} (-0.2) & 98.3 \scriptsize{$\pm$ 1.23} (-0.4) & 89.9 \scriptsize{$\pm$ 1.31} & 91.9 \scriptsize{$\pm$ 1.41} (+1.9) & 92.7 \scriptsize{$\pm$ 1.06} (+2.7) \\
\bottomrule
\end{tabular}
}
\end{table}

\subsubsection{Ablation study}

We finally ablate two design choices: the \textit{MIL aggregator} and the \textit{number of experts}.
For the former, we replace ABMIL~\citep{pmlr-v80-ilse18a} with DSMIL~\citep{li2021dual} or TransMIL~\citep{shao2021transmil} while keeping the PFMs and fusion strategy fixed.
For the latter, we vary $K$ from two to five experts to assess the effect of expert diversity.

\paragraph{MIL aggregator.}
Table~\ref{tab:mil-head-5encoder-auc-delta} reports slide-level AUC when ABMIL is replaced by DSMIL or TransMIL.
Across all three datasets, $\Delta$ remains within $\pm2.7$ percentage points of ABMIL.
This stability indicates that our fusion framework is largely agnostic to the choice of MIL aggregator.

\paragraph{Number of experts.}
Figure~\ref{fig:encoder_count_ablation} shows that increasing the number of experts is most beneficial on PANDA, where \ourmethod improves sharply when all five PFMs are used.
On C16, performance is already saturated with fewer experts and changes only marginally as $K$ increases.
On C17, the best F1 is obtained with fewer experts and including all five does not uniformly improve performance.
These trends support the ensemble view that the benefits of adding experts stem from their complementarity rather than from ensemble size alone, with diminishing returns once the feature space is well covered.

\begin{figure}[htbp]
    \centering
    \begin{subfigure}[b]{0.32\linewidth}
        \includegraphics[width=\linewidth]{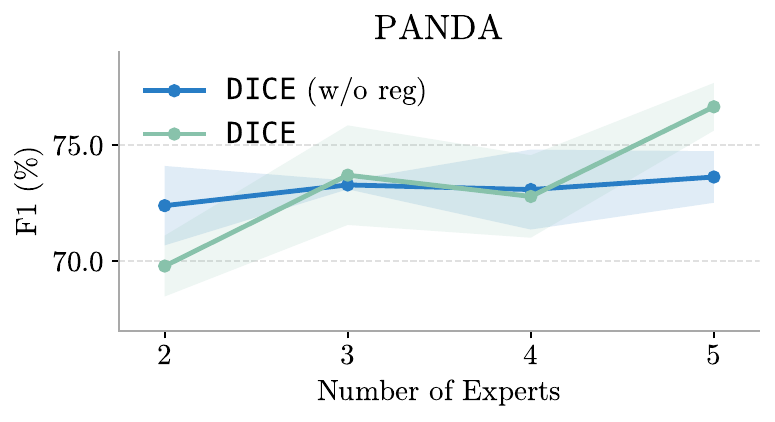}
    \end{subfigure}
    \hfill
    \begin{subfigure}[b]{0.32\linewidth}
        \includegraphics[width=\linewidth]{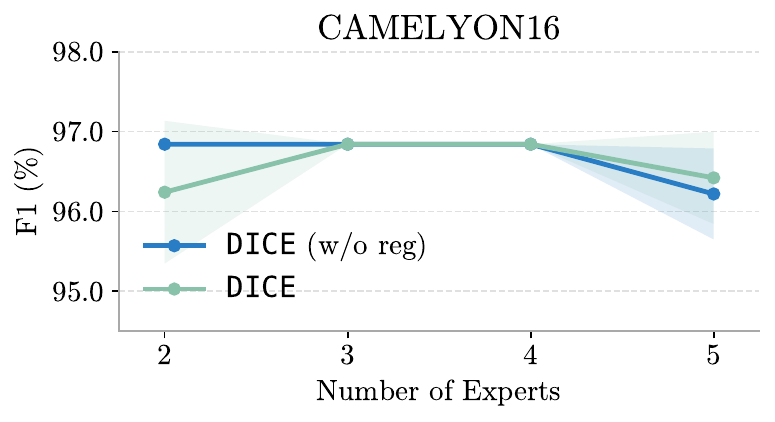}
    \end{subfigure}
    \hfill
    \begin{subfigure}[b]{0.32\linewidth}
        \includegraphics[width=\linewidth]{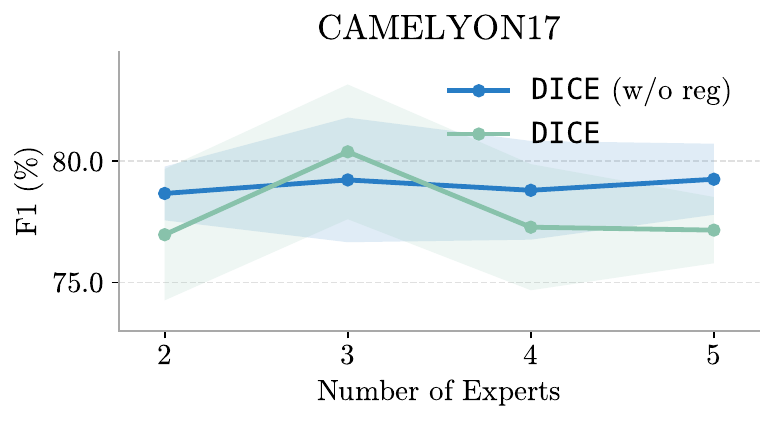}
    \end{subfigure}
    \caption{\textbf{Ablation of the number of experts.} Slide-level F1 (\%) for two to five experts. Note that experts are added in decreasing order of average single PFM test F1 across the three datasets (Virchow2 $>$ UNI2-h $>$ H-optimus-1 $>$ CONCHv1.5 $\approx$ Hibou-L).}
    \label{fig:encoder_count_ablation}
\end{figure}

\section{Conclusion}

We introduced \ourmethod, a novel plug-and-play framework that leverages pathology foundation model ensembles as a principled tool for slide-level uncertainty estimation and reliable downstream prediction.
\ourmethod couples the ensemble members through deep mutual learning and aligns their representations with the Gramian measure.
In particular, on the downstream datasets, \ourmethod accurately identifies error-prone cases, improves end-task performance compared to SOTA models, and localizes annotated regions without supervision.
However, several limitations remain.
First, combining $K$ frozen PFMs requires extracting $K$ features per slide, which is more expensive than using a single backbone.
Second, \ourmethod depends on a careful balance between alignment and diversity.
While DML and Gramian regularization improve downstream performance, excessive alignment may 
adversely affect uncertainty estimation, whereas insufficient alignment limits the benefits of fusion.
We mitigate this trade-off through scheduled regularization that gradually introduces alignment over training, but more principled approaches to setting the alignment strength remain an open question.
Finally, the number of experts $K$ is fixed a priori.
Increasing $K$ can improve robustness, but may also introduce redundancy expected to degrade performance when additional PFMs do not contribute adequately.
All in all, our work paves the way toward more reliable and uncertainty-aware pathology models, contributing to the development of systems that can eventually be integrated into clinical workflows.

\newpage

\subsection*{AI use statement}
In this work, we used generative AI tools to:
\begin{enumerate}
    \item Design or provide feedback on research methodology or experiments.
    The core ideas are all human-generated.
    \item Implement methods, especially to optimize parts of our code for efficiency.
    All code was reviewed by the authors, who take full responsibility for its correctness. 
    \item Support qualitative and thematic data analysis, especially in making plots and visualizations.
\end{enumerate}
We did not use generative AI tools to:
\begin{enumerate}
    \item Clean and reformat datasets.
    \item Interpret results.
    \item Help develop theoretical models or conceptual frameworks.
    \item Formulate mathematical claims.
    \item Provide critical ingredients for proving mathematical claims.
    \item Assist in the writing of proofs.
    \item Propose or refine hypotheses.
\end{enumerate}
These required disclosure tasks do not apply to this work:
\begin{enumerate}
    \item Generate synthetic data sets.
    \item Assist with translation.
\end{enumerate}
Additionally, we used generative AI tools to:
\begin{enumerate}
    \item Edit manuscript text to improve readability, for example, to find and correct grammar mistakes.
    \item Identify relevant literature.
    All retrieved work was manually reviewed by authors.
\end{enumerate}
We take full responsibility for the final manuscript, including all content produced with the assistance of generative AI.

\subsection*{Ethics statement}
This work uses publicly available digital pathology datasets, with all patient data anonymized.
Therefore, we see no ethical concerns arising from this work.

\subsection*{Reproducibility statement}
Our implementation uses Python 3.10 and PyTorch, and we provide the full training and evaluation pipeline as supplementary material, together with a complete list of dependencies.
All experiments were run on a single NVIDIA H100 (80GB) GPU with CUDA 12.0, and we fixed the pseudo-random seed across all stochastic components.
\newpage

\bibliography{iclr2027_conference}
\bibliographystyle{iclr2027_conference}
\newpage

\tableofcontents
\newpage

\appendix
\section{Appendix}

\subsection{Proofs}
\label{app:proofs}

\begin{proof}[Proof of Proposition~\ref{prop:jsd}]
Let $p_c^{(k)}$ denote the $c^{\text{th}}$ entry of $\mathbf{p}^{(k)}$ and $\bar{p}_c$ the $c^{\text{th}}$ entry of $\bar{\mathbf{p}}$, such that $\bar{p}_c = \frac{1}{K}\sum_{k=1}^{K} p_c^{(k)}$.
By definition
\[
u^{\mathrm{epi}}
= \mathcal{H}(\bar{\mathbf{p}}) - \frac{1}{K}\sum_{k=1}^{K}\mathcal{H}(\mathbf{p}^{(k)})
= -\sum_{c=1}^{C} \bar{p}_c \log \bar{p}_c + \frac{1}{K}\sum_{k=1}^{K}\sum_{c=1}^{C} p_c^{(k)} \log p_c^{(k)}.
\]
Substituting $\bar{p}_c$ and interchanging the summation order gives
\[
-\sum_{c=1}^{C} \bar{p}_c \log \bar{p}_c
= -\sum_{c=1}^{C}\Bigl(\frac{1}{K}\sum_{k=1}^{K} p_c^{(k)}\Bigr) \log \bar{p}_c
= -\frac{1}{K}\sum_{k=1}^{K}\sum_{c=1}^{C} p_c^{(k)} \log \bar{p}_c.
\]
Substituting back into the expression for $u^{\mathrm{epi}}$ and factoring yields
\[
u^{\mathrm{epi}}
= \frac{1}{K}\sum_{k=1}^{K}\sum_{c=1}^{C} p_c^{(k)} \bigl(\log p_c^{(k)} - \log \bar{p}_c\bigr)
= \frac{1}{K}\sum_{k=1}^{K}\sum_{c=1}^{C} p_c^{(k)} \log \frac{p_c^{(k)}}{\bar{p}_c}
= \frac{1}{K}\sum_{k=1}^{K}\mathrm{KL}\!\left(\mathbf{p}^{(k)} \,\|\, \bar{\mathbf{p}}\right).
\]

Recalling the definition of the multi-way Jensen--Shannon divergence~\citep{lin1991divergence}
\[
\mathrm{JSD}\!\left(\mathbf{p}^{(1)},\dots,\mathbf{p}^{(K)}\right)
:=
\frac{1}{K}\sum_{k=1}^{K}\mathrm{KL}\!\left(\mathbf{p}^{(k)} \,\Big\|\, \frac{1}{K}\sum_{\ell=1}^{K}\mathbf{p}^{(\ell)}\right),
\]
shows that it coincides with the expression of $u^{\mathrm{epi}}$.
Note that each term $\mathrm{KL}(\mathbf{p}^{(k)} \| \bar{\mathbf{p}})$ is non-negative by Gibbs' inequality, with equality if and only if $\mathbf{p}^{(k)} = \bar{\mathbf{p}}$.
Hence the average $u^{\mathrm{epi}}$ is non-negative, and $u^{\mathrm{epi}} = 0$ if and only if $\mathbf{p}^{(k)} = \bar{\mathbf{p}}$ for every $k$, which is equivalent to $\mathbf{p}^{(1)} = \cdots = \mathbf{p}^{(K)}$.
\end{proof}

\begin{proof}[Proof of Proposition~\ref{prop:dml_bound}]
Using Proposition~\ref{prop:jsd},
\[
u^{\mathrm{epi}} = \frac{1}{K}\sum_{k=1}^{K} \mathrm{KL}\!\left( \mathbf{p}^{(k)} \,\Big\|\, \frac{1}{K}\sum_{\ell=1}^{K}\mathbf{p}^{(\ell)} \right).
\]
Fixing $k$ and applying Jensen's inequality yields
\[ \mathrm{KL}\!\left(\mathbf{p}^{(k)} \,\Big\|\, \frac{1}{K}\sum_{\ell=1}^{K}\mathbf{p}^{(\ell)} \right) \le \frac{1}{K}\sum_{\ell=1}^{K} \mathrm{KL}\!\left( \mathbf{p}^{(k)} \,\Big\|\, \mathbf{p}^{(\ell)} \right)
\]
since KL is convex in the second argument.
Averaging over $k$ gives
\[ u^{\mathrm{epi}} \le \frac{1}{K^2} \sum_{k=1}^{K}\sum_{\ell=1}^{K} \mathrm{KL}\!\left( \mathbf{p}^{(k)} \,\Big\|\, \mathbf{p}^{(\ell)} \right).
\]
When $k=\ell$, we have $\mathrm{KL}\!\left( \mathbf{p}^{(k)} \,\Big\|\, \mathbf{p}^{(\ell)} \right) = 0$, and the diagonal terms vanish yielding
\[ u^{\mathrm{epi}} \le \frac{1}{K^2} \sum_{k\neq \ell} \mathrm{KL}\!\left( \mathbf{p}^{(k)} \,\Big\|\, \mathbf{p}^{(\ell)} \right).
\]
Finally, relabeling indices in the double sum gives
\[ \sum_{k\neq \ell} \mathrm{KL}\!\left( \mathbf{p}^{(k)} \,\Big\|\, \mathbf{p}^{(\ell)} \right) = \sum_{k\neq \ell} \mathrm{KL}\!\left( \mathbf{p}^{(\ell)} \,\Big\|\, \mathbf{p}^{(k)} \right) = (K-1)\sum_{k=1}^{K}\mathcal{L}_{\mathrm{DML}}^{(k)} = (K-1)\mathcal{L}_{\mathrm{DML}},
\]
which proves the result.
\end{proof}

\subsection{Datasets}
\label{app:datasets}

\paragraph{PANDA.}
The Prostate cANcer graDe Assessment (PANDA) challenge dataset~\citep{bulten2022artificial} (CC BY-SA-NC 4.0) comprises 10,616 H\&E-stained prostate biopsy whole-slide images acquired at two centers, the Karolinska Institute and Radboud University Medical Center.
Each slide carries an ISUP grade label in $\{0, 1, 2, 3, 4, 5\}$, defining a six-class cancer severity grading task.
As the publicly released labels are known to contain noise, we discard 603 slides flagged by the keep-list released with the first-place solution of the 2020 Kaggle competition~\citep{Yoshioka_Kaggle-PANDA-1st-place-solution_2024}, retaining 10,013 slides.
We then hold out 20\% of this cleaned cohort as a fixed test set, since the original test set is not released.
For patch-level evaluation, we leverage the pixel-level annotations released with the dataset.
Karolinska masks distinguish background, benign, and cancerous tissue, whereas, Radboud masks distinguish background, stroma, healthy epithelium, and Gleason patterns 3, 4, and 5.
We binarize both into a tumor / non-tumor map by collapsing the cancerous classes into a single tumor class.
A patch receives a positive label if it contains any tumorous pixel.

\paragraph{CAMELYON16.}
The CAMELYON16 challenge dataset~\citep{ehteshami2017diagnostic} (CC0) comprises 399 H\&E-stained whole-slide images of sentinel lymph nodes from breast cancer patients, collected at two centers.
We formulate the downstream task as a slide-level binary classification of metastasis and follow the official train/test split.
For patch-level evaluation, we leverage the pixel-level metastasis annotations, and a patch receives a positive label if it contains any annotated pixel.

\paragraph{CAMELYON17.}
The CAMELYON17 dataset~\citep{litjens20181399} (CC0) extends CAMELYON16 to a multi-center, patient-level setting, comprising 1000 H\&E-stained whole-slide images of sentinel lymph nodes from 200 breast cancer patients (five slides each).
Each slide carries a pN-stage label, which we binarize for metastasis classification.
Since slide-level labels for the official test split are not public, we adopt the annotations from~\citet{ling2025cbpln} and follow their quality-control protocol, excluding 49 slides affected by focal blurriness, poor staining, ambiguous metastatic foci, or treatment-related artifacts. 
We then build a fixed test set by assigning all patients with at least one lesion-annotated slide to the test set and randomly adding negative patients until it covers 20\% of the cohort.
The remainder is used for five-fold cross-validation, with splits stratified jointly by patient and label to prevent leakage.
For patch-level evaluation, we use lesion-level annotations from 50 slides and treat all patches from negative slides as negatives.
Finally, we apply the same tumor-area threshold as in CAMELYON16 to ensure consistency across datasets.

\subsection{Baselines}
\label{app:baselines}

\subsubsection{Pathology foundation models}
 
We describe the pathology foundation models used to instantiate our experts, which span diverse architectures, pretraining objectives, and training corpora.
 
\paragraph{Virchow2} is proposed by~\citet{zimmermann2024virchow2} as a ViT-H/14 ($\sim$632M parameters) trained with a domain-adapted DINOv2 architecture on $\sim$1.7 billion tiles sampled at multiple magnifications ($5\times$, $10\times$, $20\times$, $40\times$) from 3.1 million H\&E- or IHC-stained whole-slide images collected at the Memorial Sloan Kettering Cancer Center.
 
\paragraph{UNI2-h} is introduced by~\citet{chen2024uni} as a ViT-H/14 model trained with the DINOv2 framework on more than 200 million tiles sampled from over 350,000 H\&E- and IHC-stained whole-slide images sourced from the Mass General Brigham.
 
\paragraph{H-optimus-1} is put forth by~\citet{hoptimus1} as a 1.1B-parameter ViT-G/14 trained with DINOv2 on a proprietary dataset of several billion H\&E tiles sampled from over 1 million slides of more than 800,000 patients across 50 organs.
 
\paragraph{CONCHv1.5} is based on the CONCH model~\citep{lu2024visual}, which is a vision-language model pretrained on 1.17 million image-caption pairs in histopathology.
It is initialized from the UNI checkpoint~\citep{chen2024uni} and further trained with the CoCa objective~\citep{yu2022coca}.
 
\paragraph{Hibou-L} is built by~\citet{nechaev2024hibou} as a ViT-L/14 ($\sim$307M parameters) model trained using the DINOv2 framework on 1.2 billion patches sampled from a proprietary dataset comprising over 1 million whole-slide images.

\subsubsection{Uncertainty quantification methods}

\paragraph{Monte Carlo dropout.}
We implement Monte Carlo Dropout~\citep{gal2016dropout} on top of the single PFM baseline that reaches the highest test F1 for each dataset.
Specifically, dropout with probability $p=0.25$ is applied to the projected patch features, the attention hidden representation, and the slide embedding.
At test time, the model is kept in evaluation mode while dropout layers remain active, and we perform $15$ stochastic forward passes per slide.

\paragraph{Deep ensemble.}
We train $K=5$ MIL heads with different random seeds on the same PFM and average their posteriors at inference.

\paragraph{Evidential deep learning.}
We adapt EDL~\citep{sensoy2018evidential} to ABMIL by interpreting its outputs as non-negative evidence for a Dirichlet distribution over classes.
We use the expected cross-entropy objective with an annealed KL-to-uniform penalty and use the Dirichlet mean for prediction.

\paragraph{Spectral-normalized neural Gaussian-process.}
We adapt SNGP~\citep{liu2020simple} by applying spectral normalization to the ABMIL linear layers and replacing its classifier with a random-feature Gaussian-process output layer, which returns a predictive posterior and variance.

\paragraph{DivDis.}
We adapt DivDis~\citep{lee2023diversify} as an ABMIL model with two classifier heads, optimized with a supervised loss and a pairwise mutual-information diversity objective.
The diversity term uses training bags with their labels hidden.
At inference, we average the head posteriors.

\subsubsection{Fusion methods}

\paragraph{Early fusion.}
For each slide, patch-level features from all $K=5$ single PFMs are extracted on a shared patch grid and concatenated at the patch level.
A single MIL head is then trained under the standard supervised loss.

\paragraph{Late fusion.}
We use $K=5$ MIL heads, each independently trained on a different PFM backbone, and average their posteriors at inference.
The five PFMs match those used in \ourmethod.
 
\subsection{Evaluation metrics}
\label{app:metrics}

\paragraph{Classification.}
We report the area under the receiver operating characteristic curve (AUC) and the F1 score at the slide level.
AUC is computed from the posteriors: one-versus-rest and macro-averaged for multi-class tasks, and reduces to the standard ROC-AUC on the positive-class posterior for binary tasks.
F1 is computed from hard predictions obtained by argmax of the posteriors, with macro averaging for multi-class tasks and the standard positive-class F1 for binary tasks.

\paragraph{Calibration.}
We report the negative log-likelihood (NLL) and the Brier score on the slide-level posteriors:
\[
\mathrm{NLL} = -\frac{1}{M}\sum_{m=1}^{M} \log p_{y_m}(m),
\qquad
\mathrm{Brier} = \frac{1}{MC}\sum_{m=1}^{M} \sum_{c=1}^{C}
\bigl(p_c(m) - \mathds{1}[y_m{=}c]\bigr)^2,
\]
where $p_c(m)$ denotes the predicted probability of class $c$ for slide $m$ and $M$ is the total number of slides in the dataset.
For binary tasks, the Brier score reduces to the mean squared error between the positive-class posterior and the binary label.
Lower is better for both metrics.

\paragraph{Localization.}
Given that our datasets are pixel-level annotated, we can evaluate whether per-patch attention scores localize tumorous regions.
Since the within-slide softmax makes raw attention values depend on slide size, we min-max normalize $c_i^{(\mathrm{patch})}$ within each slide, and report the patch-level AUC of the normalized scores against the binary tumor labels constructed above.

\subsection{Implementation details}
\label{app:implementation}

\paragraph{Data preprocessing.}
Slides are tiled into non-overlapping $512{\times}512$-pixel patches at $20\times$ magnification, with tissue regions identified by the HEST segmenter~\citep{jaume2024hest} at a confidence threshold of $0.5$ and a minimum per-patch tissue proportion of $0.2$.
Then, we extract patch-level features from each pathology foundation model using the TRIDENT toolkit~\citep{zhang2025standardizing}. 

\paragraph{Model selection.}
When no official test split is available, we hold out a fixed 20\% test split and run a five-fold stratified cross-validation on the remaining slides.
Stratification is on the binary metastasis label for CAMELYON16 and CAMELYON17, and jointly on the data provider and ISUP grade for PANDA.
Models are trained for up to 50 epochs with a batch size of 16 (reduced to 8 for TransMIL due to its higher memory usage), with early stopping on validation loss, using a patience of 10 epochs and a minimum improvement threshold of $10^{-3}$.
The best epoch per fold is selected based on the minimum validation loss.
We report the mean and standard deviation across the five folds using the fixed test split.

\paragraph{Hyperparameter tuning.}
Hyperparameters are selected with Optuna~\citep{optuna_2019} using the default Tree-structured Parzen Estimator (TPE) sampler and a median pruner that begins evaluating trials after the first five are complete and prunes thereafter at every fold.
Each trial is run across all five folds with fold-level intermediate reporting, so that unpromising configurations are stopped early.
Trials are ranked by mean validation AUC across folds, with ties broken first by lower mean validation loss and then by lower standard deviation of the validation loss.
Within each fold, trials are trained for up to 20 epochs, and early stopping and best-checkpoint selection use the epoch with the lowest validation loss.
The search space is shared across experiments:
\begin{itemize}
    \item Projection dimension $d_r \in \{128, 256, 512, 1024\}$,
    \item Dropout rate $\in \{0.0, 0.1, 0.2, 0.3, 0.4\}$,
    \item Learning rate log-uniform in $[10^{-5},\, 5 \times 10^{-4}]$,
    \item MIL attention dimension $\in \{64, 128, 256\}$,
    \item MIL hidden dimension $\in \{128, 256, 512\}$, or no hidden layer,
    \item Gramian schedule $\in \{\text{linear},\text{cosine}\}$ starting at $0.01$ and ending at $0.50$.
\end{itemize}
A single hyperparameter configuration is shared across the $K$ experts of a \ourmethod or \ourmethodnoreg trial, whereas all baselines are tuned independently.
We run 20 trials for CAMELYON16 and CAMELYON17, and 10 trials for PANDA, whose larger size makes each trial substantially more expensive.
Finally, for the classification with rejection experiment, we tune the rejection threshold on the validation set by evaluating rejection rates between 5\% and 50\% at 1\% intervals.
For each rejection rate, we identify the uncertainty threshold yielding the highest F1 score, and subsequently apply it to the test set, where highly uncertain samples are deferred for secondary review.

\subsection{Computational cost analysis}
\label{app:inference_cost}

In order to investigate the effect of having $K$ PFMs, as opposed to only one, we profile the full inference pipeline on 10 randomly sampled CAMELYON16 slides using a single NVIDIA H100 SXM5 80GB GPU.
The pipeline comprises WSI preprocessing (tissue segmentation with a pretrained segmenter and patch extraction), frozen PFM feature extraction, feature loading, and MIL-head prediction.
The five PFMs in \ourmethod are executed sequentially, with WSI preprocessing performed once and shared across encoders.

\begin{table}[htbp]
    \centering
    \caption{Inference computational cost. Runtimes are in seconds per slide (mean $\pm$ standard deviation) and peak GPU memory is in GB.}
    \label{tab:inference_cost}
    \resizebox{\linewidth}{!}{%
    \begin{tabular}{lcccccc}
        \toprule
        \textbf{Method} & \shortstack{WSI\\preprocessing} & \shortstack{Feature\\extraction} & \shortstack{Feature\\loading} & \shortstack{MIL\\inference} & \shortstack{Total\\time} & \shortstack{Peak GPU\\memory} \\
        \midrule
        Virchow2 & $106.7 \pm 28.5$ & $46.5 \pm 9.5$ & $0.056 \pm 0.029$ & $0.008 \pm 0.001$ & $153.3 \pm 22.5$ & $4.35$ \\
        \ourmethod & $106.7 \pm 28.5$ & $238.3 \pm 49.3$ & $0.154 \pm 0.057$ & $0.013 \pm 0.002$ & $345.2 \pm 35.0$ & $7.43$ \\
        \bottomrule
    \end{tabular}%
    }
\end{table}

Table~\ref{tab:inference_cost} demonstrates that (i) WSI preprocessing is the dominant cost for a single PFM and is shared across all methods, (ii) feature extraction is indeed around five times more expensive for \ourmethod, though this reflects sequential extraction and the five encoders are trivially parallelizable, (iii) the feature extraction can simply be done offline and cached, which prevents a hindrance to real time use, and (iv) end-to-end inference is consequently only 2.25 times that of a single PFM, since the preprocessing overhead is paid once regardless of $K$.
In terms of memory, sequential execution leads to \ourmethod's peak GPU memory matching that of the largest single PFM. 

\subsection{Discussion of concurrent work}
\label{app:concurrent}

Unfortunately, none of the mentioned methods in Section~\ref{sec:setup} is directly comparable to ours, for two distinct reasons.
First, ELF~\citep{luo2025ensemble} and FuseCPath~\citep{yang2025fusion} are unpublished preprints whose released artifacts are insufficient to reproduce their pipelines.
Specifically, the ELF pretrained weights are unavailable and the FuseCPath repository explicitly states that their code is only partially complete, with the remainder deferred until paper acceptance.
FM2~\citep{yu2025fm2} has, to our knowledge, no public code or weights at the time of submission.
Second, PICTURE~\citep{zhao2025uncertainty} is only applicable to glioblastoma diagnosis.
Its prototype-guided inference step requires expert-curated reference images that distinguish glioblastoma from its mimics, and constructing analogous prototype sets for prostate ISUP grading or breast lymph-node metastasis would require dedicated clinical curation that is outside the scope of our paper.

\subsection{Additional related work}
\label{app:related}

\paragraph{Multiple instance learning.}
A standard way to obtain slide-level representations from patch-level PFMs is multiple instance learning (MIL), where a slide is represented as a \textit{bag} of patch-level \textit{instances}, and only a bag label is available.
While the fundamental attention-based MIL method (ABMIL)~\citep{pmlr-v80-ilse18a} and its extensions, such as CLAM~\citep{lu2021data}, DSMIL~\citep{li2021dual}, and TransMIL~\citep{shao2021transmil} excel at weakly supervised classification, their clinical utility is limited by a lack of explicit bag-level confidence.
To tackle this, BayesMIL~\citep{cui2023bayesmil} learns a posterior over attention weights via variational inference, but its reliance on multiple stochastic forward passes (typically on the order of 15--20) and associated optimization challenges hinder its general use.

\subsection{Additional results}
\label{app:results}

This section reports additional metrics that complement the results in the main paper.
Table~\ref{tab:abmil-classification-performance} reports AUC and F1 for slide-level classification, while Table~\ref{tab:abmil-calibration-performance} reports NLL and Brier score for slide-level calibration.
Across all these metrics, the results are consistent with our main findings.

\begin{table}[htbp]
\caption{\textbf{Slide-level classification performance on PANDA, C16, and C17.} We report slide-level classification with AUC (\%) and F1 (\%). Values are mean $\pm$ std across five folds on the held-out test set. MCD and DE are built on the single PFM with the highest test F1 per dataset, while fusion and \ourmethod{} variants use all five PFMs. Bold and underline indicate the best and second-best per column.}
\label{tab:abmil-classification-performance}
\centering
\resizebox{\textwidth}{!}{%
\begin{tabular}{lcccccc}
\toprule
 & \multicolumn{2}{c}{\textbf{PANDA}} & \multicolumn{2}{c}{\textbf{CAMELYON16}} & \multicolumn{2}{c}{\textbf{CAMELYON17}} \\
\cmidrule(lr){2-3}\cmidrule(lr){4-5}\cmidrule(lr){6-7}
 & AUC ($\uparrow$) & F1 ($\uparrow$) & AUC ($\uparrow$) & F1 ($\uparrow$) & AUC ($\uparrow$) & F1 ($\uparrow$) \\
\midrule
\multicolumn{7}{l}{\textbf{Single PFMs}} \\
Virchow2 & 95.7 \scriptsize{$\pm$ 0.24} & 72.8 \scriptsize{$\pm$ 1.51} & \underline{99.3 \scriptsize{$\pm$ 0.41}} & \underline{96.3 \scriptsize{$\pm$ 0.59}} & 90.3 \scriptsize{$\pm$ 1.52} & 79.8 \scriptsize{$\pm$ 2.24} \\
UNI2-h & 95.1 \scriptsize{$\pm$ 0.12} & 69.7 \scriptsize{$\pm$ 0.82} & 98.7 \scriptsize{$\pm$ 0.49} & 95.7 \scriptsize{$\pm$ 1.96} & 89.2 \scriptsize{$\pm$ 2.22} & \textbf{80.6 \scriptsize{$\pm$ 3.47}} \\
H-optimus-1 & 95.3 \scriptsize{$\pm$ 0.12} & 71.2 \scriptsize{$\pm$ 0.87} & 97.9 \scriptsize{$\pm$ 0.83} & 94.4 \scriptsize{$\pm$ 1.15} & 85.0 \scriptsize{$\pm$ 2.60} & 78.2 \scriptsize{$\pm$ 2.27} \\
CONCHv1.5 & 94.3 \scriptsize{$\pm$ 0.09} & 68.7 \scriptsize{$\pm$ 1.32} & 96.1 \scriptsize{$\pm$ 0.82} & 94.5 \scriptsize{$\pm$ 0.99} & \underline{91.0 \scriptsize{$\pm$ 2.09}} & 78.0 \scriptsize{$\pm$ 3.48} \\
Hibou-L & 94.8 \scriptsize{$\pm$ 0.07} & 69.9 \scriptsize{$\pm$ 0.88} & 96.2 \scriptsize{$\pm$ 1.22} & 92.4 \scriptsize{$\pm$ 0.90} & 85.6 \scriptsize{$\pm$ 3.29} & 79.2 \scriptsize{$\pm$ 3.86} \\
\midrule
\multicolumn{7}{l}{\textbf{UQ Methods}} \\
MC Dropout & 95.7 \scriptsize{$\pm$ 0.26} & 72.3 \scriptsize{$\pm$ 1.69} & 99.1 \scriptsize{$\pm$ 0.87} & 95.2 \scriptsize{$\pm$ 0.60} & 88.5 \scriptsize{$\pm$ 2.99} & 79.0 \scriptsize{$\pm$ 5.61} \\
Deep Ensemble & 95.9 \scriptsize{$\pm$ 0.16} & 73.8 \scriptsize{$\pm$ 0.11} & \underline{99.3 \scriptsize{$\pm$ 0.58}} & 95.8 \scriptsize{$\pm$ 0.00} & 90.2 \scriptsize{$\pm$ 2.17} & 79.5 \scriptsize{$\pm$ 2.38} \\
EDL & 92.8 \scriptsize{$\pm$ 0.10} & 61.4 \scriptsize{$\pm$ 2.35} & \underline{99.3 \scriptsize{$\pm$ 0.53}} & 95.8 \scriptsize{$\pm$ 1.00} & 87.5 \scriptsize{$\pm$ 1.89} & 80.2 \scriptsize{$\pm$ 1.11} \\
SNGP & 95.8 \scriptsize{$\pm$ 0.13} & \underline{74.4 \scriptsize{$\pm$ 1.14}} & 98.3 \scriptsize{$\pm$ 0.91} & 96.0 \scriptsize{$\pm$ 0.49} & 81.1 \scriptsize{$\pm$ 1.42} & 69.1 \scriptsize{$\pm$ 1.65} \\
DivDis & 95.3 \scriptsize{$\pm$ 0.11} & 71.6 \scriptsize{$\pm$ 1.24} & \textbf{99.5 \scriptsize{$\pm$ 0.20}} & 95.9 \scriptsize{$\pm$ 0.67} & 88.2 \scriptsize{$\pm$ 3.00} & 79.0 \scriptsize{$\pm$ 5.92} \\
\midrule
\multicolumn{7}{l}{\textbf{Fusion Methods}} \\
Early Fusion & 96.0 \scriptsize{$\pm$ 0.19} & 73.6 \scriptsize{$\pm$ 1.02} & 98.1 \scriptsize{$\pm$ 0.95} & 95.7 \scriptsize{$\pm$ 1.75} & 90.8 \scriptsize{$\pm$ 1.33} & \underline{80.3 \scriptsize{$\pm$ 3.54}} \\
Late Fusion & 96.0 \scriptsize{$\pm$ 0.15} & 73.4 \scriptsize{$\pm$ 0.81} & 98.3 \scriptsize{$\pm$ 0.71} & 96.2 \scriptsize{$\pm$ 0.57} & 90.4 \scriptsize{$\pm$ 1.82} & 78.9 \scriptsize{$\pm$ 2.99} \\
\midrule
\multicolumn{7}{l}{\textbf{Ours}} \\
\rowcolor{gray!20} \ourmethodnoreg & \underline{96.1 \scriptsize{$\pm$ 0.27}} & 73.6 \scriptsize{$\pm$ 1.11} & 98.5 \scriptsize{$\pm$ 0.33} & 96.2 \scriptsize{$\pm$ 0.57} & \textbf{92.4 \scriptsize{$\pm$ 0.70}} & 79.3 \scriptsize{$\pm$ 1.46} \\
\rowcolor{gray!20} \ourmethod & \textbf{96.7 \scriptsize{$\pm$ 0.20}} & \textbf{76.6 \scriptsize{$\pm$ 1.02}} & 98.7 \scriptsize{$\pm$ 0.48} & \textbf{96.4 \scriptsize{$\pm$ 0.58}} & 89.9 \scriptsize{$\pm$ 1.31} & 77.2 \scriptsize{$\pm$ 1.36} \\
\bottomrule
\end{tabular}
}
\end{table}

\begin{table}[htbp]
\caption{\textbf{Slide-level calibration performance on PANDA, C16, and C17.} We report slide-level calibration with NLL ($\times 100$) and Brier score ($\times 100$). Values are mean $\pm$ std across five folds on the held-out test set. MCD and DE are built on the single PFM with the highest test F1 per dataset, while fusion and \ourmethod{} variants use all five PFMs. Bold and underline indicate the best and second-best per column.}
\label{tab:abmil-calibration-performance}
\centering
\resizebox{\textwidth}{!}{%
\begin{tabular}{lcccccc}
\toprule
 & \multicolumn{2}{c}{\textbf{PANDA}} & \multicolumn{2}{c}{\textbf{CAMELYON16}} & \multicolumn{2}{c}{\textbf{CAMELYON17}} \\
\cmidrule(lr){2-3}\cmidrule(lr){4-5}\cmidrule(lr){6-7}
 & NLL ($\downarrow$) & Brier ($\downarrow$) & NLL ($\downarrow$) & Brier ($\downarrow$) & NLL ($\downarrow$) & Brier ($\downarrow$) \\
\midrule
\multicolumn{7}{l}{\textbf{Single PFMs}} \\
Virchow2 & 58.0 \scriptsize{$\pm$ 1.30} & 5.1 \scriptsize{$\pm$ 0.14} & 14.8 \scriptsize{$\pm$ 2.44} & \textbf{2.4 \scriptsize{$\pm$ 0.37}} & 58.6 \scriptsize{$\pm$ 11.55} & 14.6 \scriptsize{$\pm$ 1.28} \\
UNI2-h & 60.0 \scriptsize{$\pm$ 0.55} & 5.5 \scriptsize{$\pm$ 0.08} & 20.6 \scriptsize{$\pm$ 8.26} & 3.0 \scriptsize{$\pm$ 1.28} & 60.2 \scriptsize{$\pm$ 10.49} & 14.6 \scriptsize{$\pm$ 2.05} \\
H-optimus-1 & 59.1 \scriptsize{$\pm$ 0.92} & 5.4 \scriptsize{$\pm$ 0.12} & 18.8 \scriptsize{$\pm$ 3.38} & 3.5 \scriptsize{$\pm$ 0.43} & 65.4 \scriptsize{$\pm$ 12.87} & 16.1 \scriptsize{$\pm$ 1.54} \\
CONCHv1.5 & 66.3 \scriptsize{$\pm$ 0.80} & 5.9 \scriptsize{$\pm$ 0.06} & 16.9 \scriptsize{$\pm$ 3.51} & 3.7 \scriptsize{$\pm$ 0.70} & 59.0 \scriptsize{$\pm$ 11.16} & 15.4 \scriptsize{$\pm$ 2.39} \\
Hibou-L & 62.6 \scriptsize{$\pm$ 0.72} & 5.6 \scriptsize{$\pm$ 0.04} & 26.7 \scriptsize{$\pm$ 4.80} & 5.3 \scriptsize{$\pm$ 0.45} & 65.9 \scriptsize{$\pm$ 7.82} & 15.5 \scriptsize{$\pm$ 2.25} \\
\midrule
\multicolumn{7}{l}{\textbf{UQ Methods}} \\
MC Dropout & 56.3 \scriptsize{$\pm$ 1.70} & 5.1 \scriptsize{$\pm$ 0.15} & 15.9 \scriptsize{$\pm$ 5.22} & 3.0 \scriptsize{$\pm$ 0.39} & 68.6 \scriptsize{$\pm$ 29.89} & 15.5 \scriptsize{$\pm$ 3.63} \\
Deep Ensemble & 53.8 \scriptsize{$\pm$ 0.95} & \underline{4.9 \scriptsize{$\pm$ 0.07}} & 12.1 \scriptsize{$\pm$ 1.72} & \textbf{2.4 \scriptsize{$\pm$ 0.11}} & 60.3 \scriptsize{$\pm$ 11.22} & 14.8 \scriptsize{$\pm$ 1.88} \\
EDL & 83.6 \scriptsize{$\pm$ 1.23} & 6.8 \scriptsize{$\pm$ 0.07} & 16.1 \scriptsize{$\pm$ 3.77} & 2.8 \scriptsize{$\pm$ 0.46} & 69.5 \scriptsize{$\pm$ 13.98} & 15.7 \scriptsize{$\pm$ 1.52} \\
SNGP & 56.9 \scriptsize{$\pm$ 1.41} & 5.0 \scriptsize{$\pm$ 0.05} & 16.8 \scriptsize{$\pm$ 1.10} & 3.4 \scriptsize{$\pm$ 0.28} & 79.6 \scriptsize{$\pm$ 8.14} & 20.9 \scriptsize{$\pm$ 1.23} \\
DivDis & 59.7 \scriptsize{$\pm$ 1.17} & 5.3 \scriptsize{$\pm$ 0.12} & 14.4 \scriptsize{$\pm$ 3.86} & 2.9 \scriptsize{$\pm$ 0.62} & 73.1 \scriptsize{$\pm$ 31.20} & 15.5 \scriptsize{$\pm$ 3.43} \\
\midrule
\multicolumn{7}{l}{\textbf{Fusion Methods}} \\
Early Fusion & 54.5 \scriptsize{$\pm$ 0.71} & 5.0 \scriptsize{$\pm$ 0.12} & 19.6 \scriptsize{$\pm$ 4.64} & 3.2 \scriptsize{$\pm$ 1.07} & 57.5 \scriptsize{$\pm$ 10.67} & \underline{14.3 \scriptsize{$\pm$ 1.99}} \\
Late Fusion & \underline{53.3 \scriptsize{$\pm$ 1.20}} & \underline{4.9 \scriptsize{$\pm$ 0.12}} & \underline{11.7 \scriptsize{$\pm$ 2.56}} & \underline{2.6 \scriptsize{$\pm$ 0.23}} & \underline{54.2 \scriptsize{$\pm$ 8.65}} & 14.6 \scriptsize{$\pm$ 1.71} \\
\midrule
\multicolumn{7}{l}{\textbf{Ours}} \\
\rowcolor{gray!20} \ourmethodnoreg & 54.6 \scriptsize{$\pm$ 2.97} & 5.0 \scriptsize{$\pm$ 0.24} & 12.2 \scriptsize{$\pm$ 1.06} & 2.9 \scriptsize{$\pm$ 0.25} & \textbf{50.5 \scriptsize{$\pm$ 2.98}} & \textbf{13.8 \scriptsize{$\pm$ 0.60}} \\
\rowcolor{gray!20} \ourmethod & \textbf{48.9 \scriptsize{$\pm$ 1.68}} & \textbf{4.5 \scriptsize{$\pm$ 0.15}} & \textbf{11.5 \scriptsize{$\pm$ 0.91}} & 2.7 \scriptsize{$\pm$ 0.15} & 54.6 \scriptsize{$\pm$ 3.04} & 15.5 \scriptsize{$\pm$ 0.68} \\
\bottomrule
\end{tabular}
}
\end{table}

\clearpage

\begin{figure}[htbp]
    \centering
    \begin{subfigure}[b]{0.29\linewidth}
        \raisebox{6pt}{\includegraphics[width=\linewidth]{figures/method_risk_coverage_PANDA_5enc_predictive_entropy_uq_methods.pdf}}
    \end{subfigure}
    \hfill
    \begin{subfigure}[b]{0.70\linewidth}
    \includegraphics[width=\linewidth]{figures/panda_abmil_methods_5enc_predictive_entropy_average_probs_correctness_box_uq_methods.pdf}
    \end{subfigure}
    \\[0.5em]
    \begin{subfigure}[b]{0.29\linewidth}
        \raisebox{6pt}
        {\includegraphics[width=\linewidth]{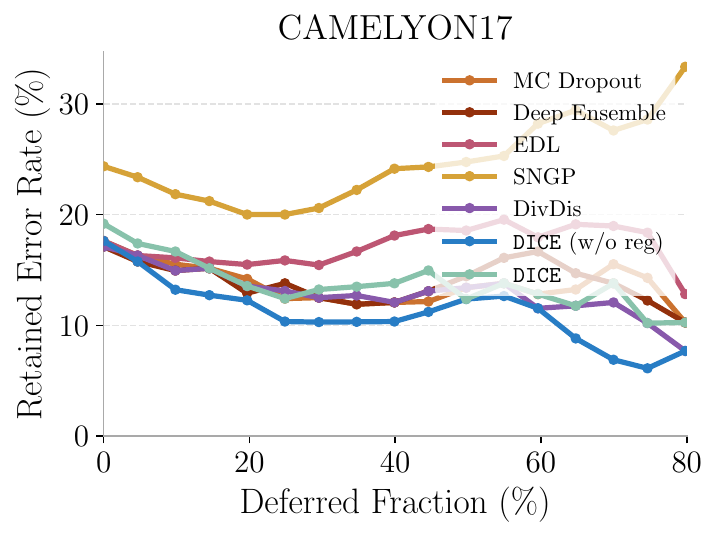}}
    \end{subfigure}
    \hfill
    \begin{subfigure}[b]{0.70\linewidth}
        \includegraphics[width=\linewidth]{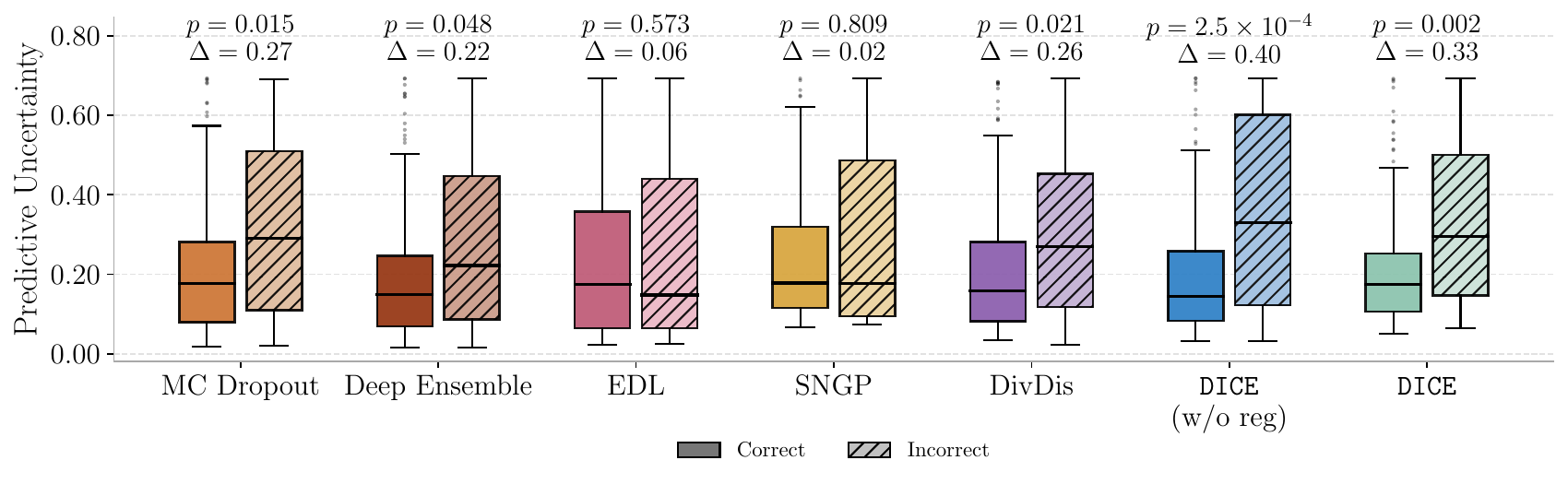}
    \end{subfigure}
    \caption{Analogue of Figure~\ref{fig:slide_uncertainty_panda}.
    On C17, \ourmethodnoreg achieves the strongest separation between correct and incorrect predictions ($\Delta = 0.40$, $p = 2.5 \times 10^{-4}$), outperforming all UQ baselines.
    The deferral curves likewise show \ourmethodnoreg achieving the lowest retained error across most of the deferral range.}
    \label{fig:slide_uncertainty_all}
\end{figure}

\begin{figure}[htbp]
    \centering
    \begin{subfigure}[b]{0.49\linewidth}
        \includegraphics[width=\linewidth]{figures/panda_uq_methods_predictive_entropy_f1_validation_optimized_together.pdf}
    \end{subfigure}
    \hfill
    \begin{subfigure}[b]{0.49\linewidth}
        \raisebox{2pt}{\includegraphics[width=\linewidth]{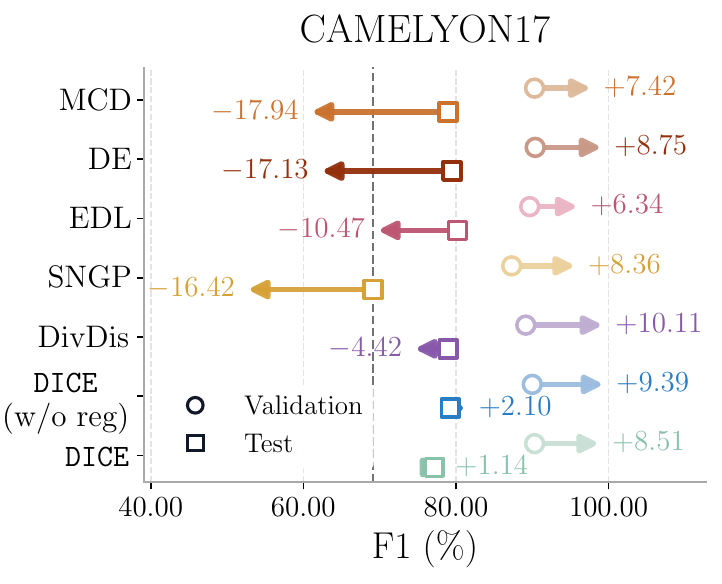}}
    \end{subfigure}
    \caption{Per-dataset analogue of Figure~\ref{fig:threshold_rejection_panda}.
    On C17, all methods improve on validation, however, the selected threshold transfers poorly for the UQ baselines.
    Only the \ourmethod variants retain positive gains.}
    \label{fig:threshold_rejection_all}
\end{figure}

\begin{figure}[htbp]
    \centering
    \begin{subfigure}[b]{\linewidth}
        \includegraphics[width=\linewidth]{figures/camelyon16_to_camelyon17_abmil_methods_5enc_predictive_entropy_average_probs_ood_generalization_box_uq_methods.pdf}
    \end{subfigure}
    \hfill
    \begin{subfigure}[b]{\linewidth}
        \includegraphics[width=\linewidth]{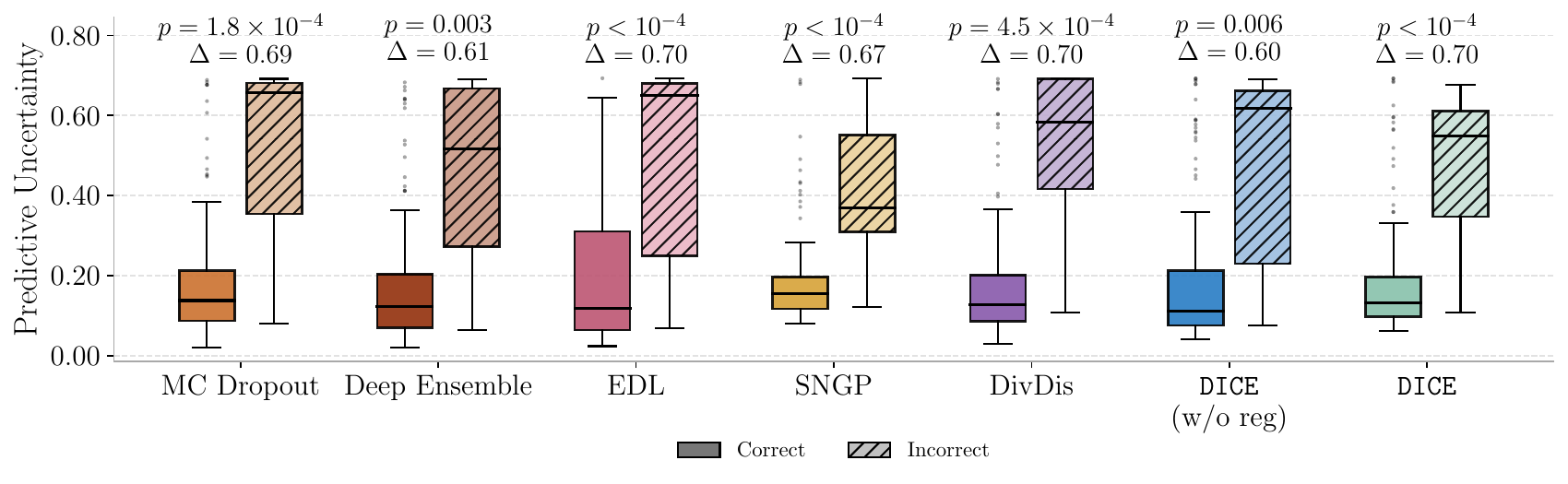}
    \end{subfigure}
    \caption{Predictive uncertainty distributions for correct and incorrect predictions under cohort shift.
    \textit{Top:} CAMELYON16 $\rightarrow$ CAMELYON17 (400 slides for training).
    \textit{Bottom:} CAMELYON17 $\rightarrow$  CAMELYON16 (1000 slides for training).
    Baselines separate the two distributions more clearly when trained on a larger cohort, but \ourmethod variants achieve the best separation in both directions.
    Annotations as in Figure~\ref{fig:slide_uncertainty_panda}.}
    \label{fig:cross_cohort}
\end{figure}

\clearpage

\begin{figure}[htbp]
    \centering
    \begin{subfigure}[b]{\linewidth}
        \includegraphics[width=\linewidth]{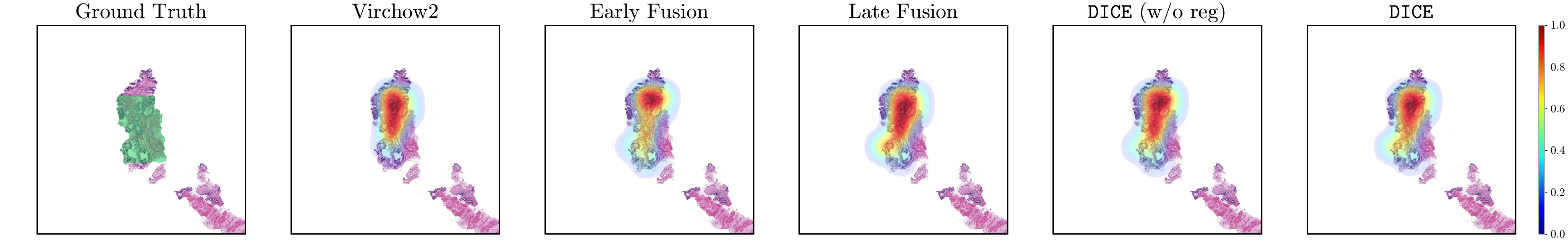}\\[0.3em]
        \includegraphics[width=\linewidth]{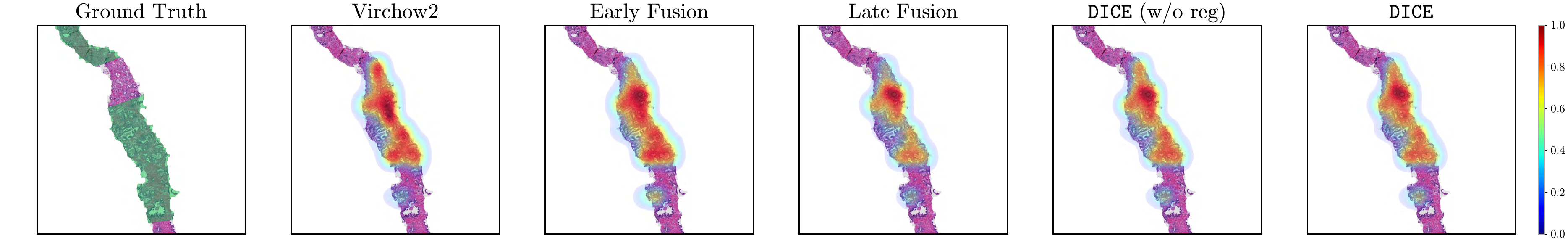}\\[0.3em]
        \includegraphics[width=\linewidth]{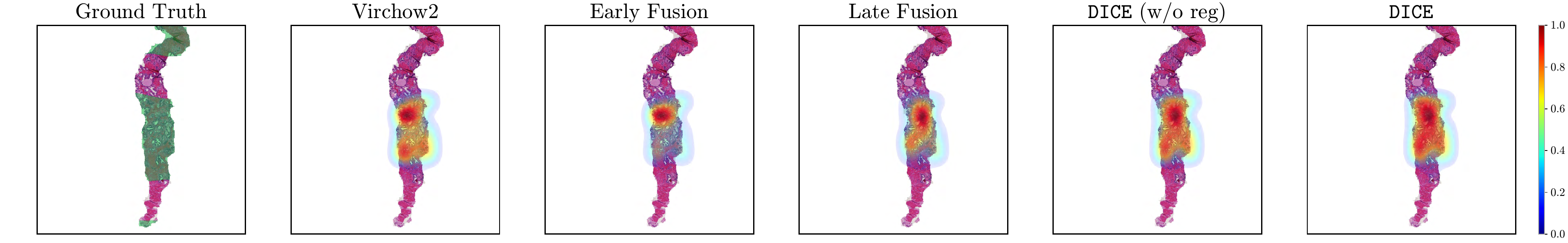}
        \caption{PANDA}
    \end{subfigure}\\[1em]
    \begin{subfigure}[b]{\linewidth}
        \includegraphics[width=\linewidth]{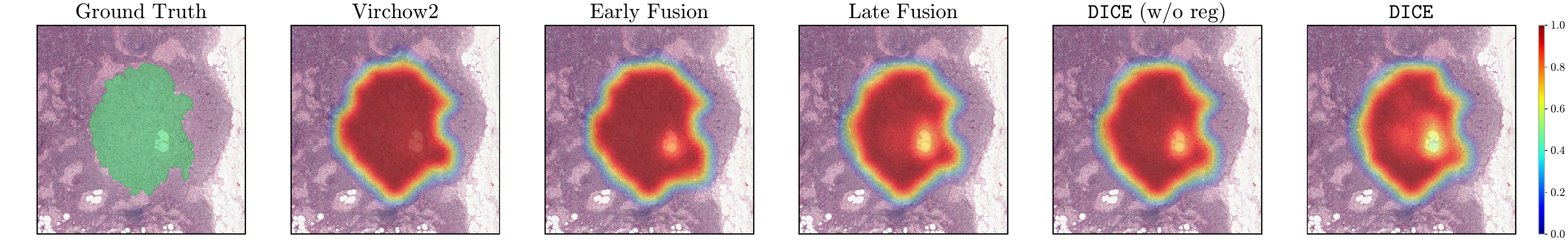}\\[0.3em]
        \includegraphics[width=\linewidth]{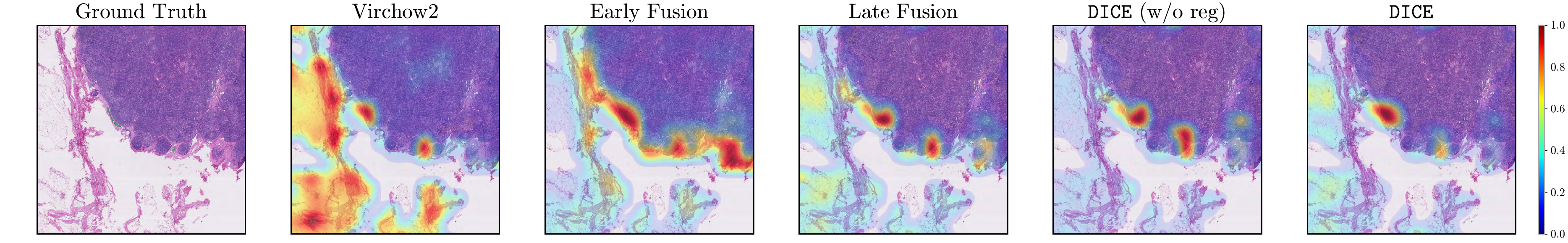}\\[0.3em]
        \includegraphics[width=\linewidth]{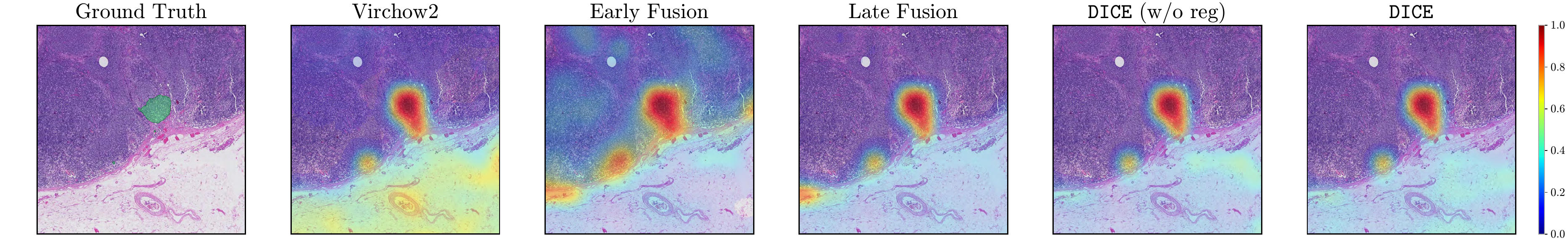}
        \caption{CAMELYON16}
    \end{subfigure}\\[1em]
    \begin{subfigure}[b]{\linewidth}
        \includegraphics[width=\linewidth]{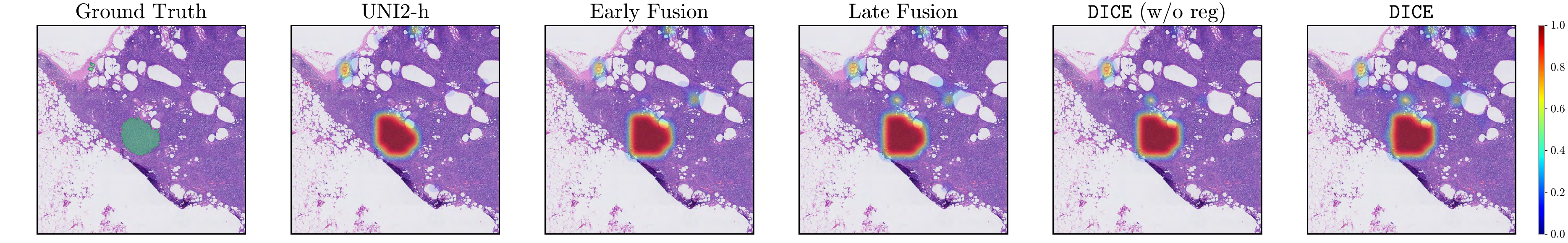}\\[0.3em]
        \includegraphics[width=\linewidth]{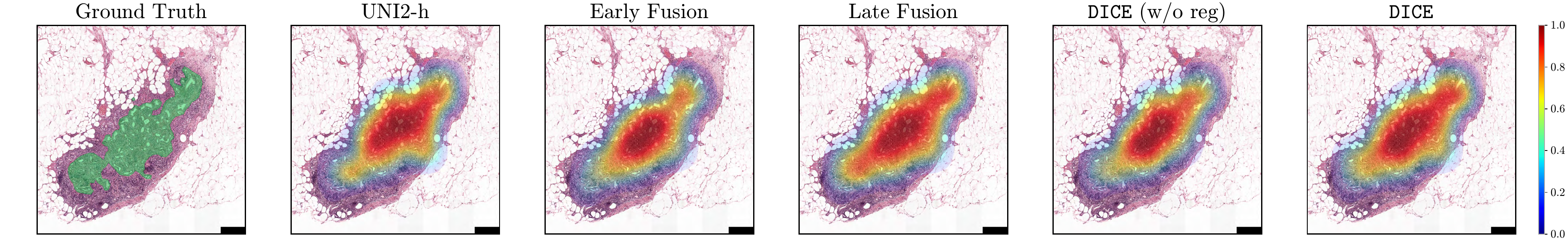}\\[0.3em]
        \includegraphics[width=\linewidth]{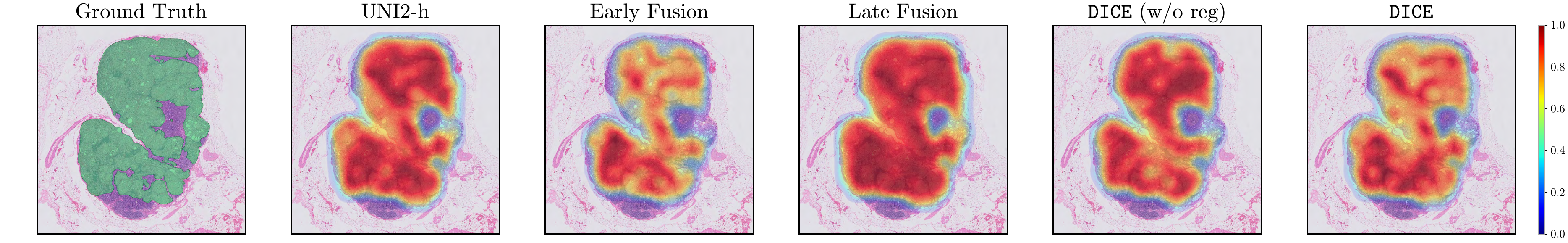}
        \caption{CAMELYON17}
    \end{subfigure}
    \caption{Additional visualizations across all datasets.
    Analogue of Figure~\ref{fig:attention_score_camelyon16}.}
    \label{fig:attention_appendix}
\end{figure}

\end{document}